\documentclass{article}
\usepackage{ijcai19}
\usepackage{times}
\usepackage{soul}
\usepackage{url}
\usepackage[hidelinks]{hyperref}
\usepackage[utf8]{inputenc}
\usepackage[small]{caption}
\usepackage{graphicx}
\usepackage{amsmath}
\usepackage{booktabs}
\usepackage{algorithm}
\usepackage{algorithmic}
\usepackage{times}  
\usepackage{helvet}  
\usepackage{courier}  
\usepackage{url}  
\usepackage{graphicx}  
\usepackage[utf8]{inputenc} 
\usepackage[T1]{fontenc}    
\usepackage{hyperref}       
\usepackage{url}            
\usepackage{booktabs}       
\usepackage{amsfonts}       
\usepackage{nicefrac}       
\usepackage{diagbox}
\usepackage{microtype}      
\usepackage{amsthm}
\usepackage{amssymb}
\usepackage{amsmath}
\usepackage{color}
\usepackage{subfig}
\usepackage{sidecap}
\usepackage[font=small]{caption}
\usepackage{wrapfig}
\usepackage{bm}
\usepackage{comment}
\usepackage{paralist}

\usepackage{algorithmic}
\usepackage{algorithm}
\usepackage{multirow}

\usepackage[titletoc,toc,title]{appendix}

\newcolumntype{L}[1]{>{\raggedright\let\newline\\\arraybackslash\hspace{0pt}}m{#1}}
\newcolumntype{C}[1]{>{\centering\let\newline\\\arraybackslash\hspace{0pt}}m{#1}}
\newcolumntype{R}[1]{>{\raggedleft\let\newline\\\arraybackslash\hspace{0pt}}m{#1}}
\newcommand{\myparagraph}[1]{\paragraph{#1}}

\newcommand{\real}{{\mathbb R}}

\newcommand{\note}[1]{{\it \color{red} [#1]}}

\newcommand{\bx}{\bm{x}}
\newcommand{\by}{\bm{y}}

\newcommand{\be}{\ensuremath{{\bm e}}}

\newcommand{\bW}{\ensuremath{{\bm W}}}

\newcommand{\bS}{\ensuremath{{\bm S}}}

\newcommand{\cR}{\ensuremath{\mathcal R}}
\newcommand{\set}[1]{\ensuremath{\mathcal #1}}

 \def\Real{{\mathbb{R}}}

 \def\cC{{\mathcal{C}}}

 \def\bI{{\bm{I}}}
 \def\bi{{\bm{i}}}
 \def\btheta{{\boldsymbol{\theta}}}
 
 \def\bomega{{\boldsymbol{\omega}}}

 \def\11{{\textbf{1}}}

 \def\cP{{\mathcal{P}}}
 \def\cN{{\mathcal{N}}}

\def\bb{{\bm{b}}}

\def\bX{{\bm{X}}}
\def\bY{{\bm{Y}}}

\newcommand{\beq}{\begin{eqnarray*}}
\newcommand{\eeq}{\end{eqnarray*}}
\newcommand{\beqn}{\begin{eqnarray}}
\newcommand{\eeqn}{\end{eqnarray}}
\newcommand{\bemn}{\begin{multiline}}
\newcommand{\eemn}{\end{multiline}}

\newtheorem{theorem}{Theorem}

\newtheorem{prop}{Proposition}
\newtheorem{lemma}{Lemma}

\def\cM{{\cal M}}

\def\cR{{\cal R}}

\def\cA{{\cal A}}
\def\cX{{\cal X}}
\def\cN{{\cal N}}

\def\cI{{\cal I}}

\def\cC{{\cal C}}
\def\cR{{\cal R}}

\def\cP{{\cal P}}
\def\cS{{\cal S}}
\def\bW{{\bm W}}

\def\ovt{\overline{\theta}}

\def\tcS{\mathcal{S}}
\def\tcM{\mathcal{M}}

\newcommand{\tbx}{\bm{x}}
\newcommand{\tby}{\bm{y}}
\newcommand{\tbX}{\bm{X}}

\newcommand{\tx}{x}

\def\tr{r}

\def\ts{s}

\def\ta{a}

\def\tpi{\pi}
\def\tcA{\mathcal{A}}
\def\tcP{\mathcal{P}}

\def\tQ{Q}
\def\tV{V}
\def\tcR{\mathcal{R}}
\def\tgam{\gamma}

\title{Solving Continual Combinatorial Selection via Deep Reinforcement Learning}

\author{
Hyungseok Song$^1$\footnote{Contact Author: 7590sok@gmail.com}\and
Hyeryung Jang $^2$\and
Hai H. Tran$^1$\and
Se-eun Yoon$^1$\and \\
Kyunghwan Son$^1$\and 
Donggyu Yun$^3$\and 
Hyoju Chung$^3$\and
Yung Yi$^1$\\
\affiliations
$^1$School of Electrical Engineering, KAIST, Daejeon, South Korea\\
$^2$Informatics, King’s College London, London, United Kingdom\\
$^3$Naver Corporation, Seongnam, South Korea\\
}

\begin{document}

\maketitle

\renewcommand{\baselinestretch}{0.97}
\begin{abstract}

        We consider the Markov Decision Process (MDP) of selecting a subset of items at each step, termed the Select-MDP (S-MDP). The large state and action spaces of S-MDPs make them intractable to solve with typical reinforcement learning (RL) algorithms 
        especially when 
         the number of items is huge. In this paper, we present a deep RL algorithm to solve this issue by adopting the following key ideas. First, we convert the original S-MDP into an \textit{Iterative Select-MDP (IS-MDP)}, which is equivalent to the S-MDP in terms of optimal actions. IS-MDP decomposes a joint action of selecting $K$ items simultaneously into $K$ iterative selections resulting in the decrease of actions at the expense of an exponential increase of states. 
         Second, we overcome this state space explosion by exploiting a special symmetry in IS-MDPs with novel weight shared Q-networks, which provably maintain sufficient expressive power. 
        Various experiments demonstrate that our approach works well even when the item space is large and that it scales to environments with item spaces different from those used in training.
\end{abstract}


\section{Introduction}
\label{sec:intro}
Imagine yourself managing a football team in a league of many matches.
Your goal is to maximize the total number of winning matches during the league.
For each match, you decide a lineup (\textit{action}$:\tilde{a}$) by selecting $K$ players among $N$ candidates to participate in it and allocating one of $C$ positions (\textit{command}$:c$) to each of them, with possible duplication.
You can observe a collection (\textit{state}$:\tilde{s}$) of the current status (\textit{information}$:i_{n}$) of each candidate player (\textit{item:} $n$).
During the match, you cannot supervise anymore until you receive the result (\textit{reward}$:\tilde{r}$), as well as the changed collection of the status (\textit{next state}$:\tilde{s}'$) of $N$ players which are stochastically determined by a transition probability function $\tilde{\cP}$.
In order to win the long league, you should pick a proper combination of the selected players and their positions to achieve not only a myopic result of the following match but also to consider a long-term plan such as the rotation of the members.
We model an MDP for these kinds of problems, termed  \textit{Select-MDP} (S-MDP), where an agent needs to make combinatorial selections sequentially.

There are many applications that can be formulated as an S-MDP including recommendation systems \cite{ricci2015recommender,zheng2018drn}, contextual combinatorial semi-bandits \cite{qin2014contextual,li2016contextual}, mobile network scheduling \cite{kushner2004convergence}, and fully-cooperative multi-agent systems controlled by a centralized agent \cite{usunier2016episodic} (when $N=K$). However, learning a good policy is challenging because the state and action spaces increase exponentially in $K$ and $N$. 
For example, our experiment shows that the vanilla DQN \cite{mnih2015human} proposed to tackle the large state space issue fails to learn the Q-function in our test environment of $N=50$, even for the simplest case of $C=1,$ $K=1$. This motivates the research on a scalable RL algorithm for tasks modeled by an S-MDP.

In this paper, we present a novel DQN-based RL algorithm for S-MDPs
by adopting a synergic combination of the following two design ideas:

\smallskip
\begin{compactenum}[D1.]
\item For a given S-MDP, we convert it into a divided but equivalent one, called \textit{Iterative Select-MDP} (IS-MDP), where the agent iteratively selects an (item, command) pair one by one during $K$ steps rather than selecting all at once.
IS-MDP significantly relieves the complexity of the joint action space per state in S-MDP; the agent only needs to evaluate $KNC$ actions during $K$ consecutive steps in IS-MDP, while it considers $\binom{N}{K} C^{K}$ actions for each step in S-MDP. We design $K$-cascaded deep Q-networks for IS-MDP, where each Q-network selects an item with an assigned command respectively while considering the selections by previous cascaded networks.

\item
Although we significantly reduce per-state action space in IS-MDP, 
the state space is still large as $N$ or $K$ grows.  
To have scalable and fast training, we consider two levels of weight parameter sharing for Q-networks: intra-sharing (I-Sharing) and unified-sharing (U-Sharing). 
In pactice, we propose to use a mixture of I- and U-sharing, which we call progressive sharing (P-sharing), by starting from a single parameter set as in U-sharing and then progressively increasing the number of parameter sets, approaching to that of I-sharing.



\end{compactenum}
\smallskip

The superiority of our ideas is discussed and evaluated in two ways. First, 
despite the drastic parameter reduction, we theorectically  
claim that I-sharing does not hurt the expressive power for IS-MDP by proving (i) \textit{relative local optimality} and (ii) \textit{universality} of I-sharing. Note that this analytical result is not limited to a Q-function approximator in RL, but is also applied to any neural network with parameter sharing in other contexts such as supervised learning. 
Second, we evaluate our approach on two self-designed S-MDP environments (circle selection and selective predator-prey) and observe a significantly high performance improvement,
especially with large $N$ (e.g., $N=200$), over other baselines. Moreover, the trained parameters can generalize to other environments
of much larger item sizes without additional training, where we use the trained parameters in $N=50$ for those in $N=200.$



\subsection{Related Work}
\paragraph{Combinatorial Optimization via RL}
Recent works on deep RL have been solving NP-hard combinatorial optimization problems on graphs \cite{dai2017learning}, Traveling Salesman problems \cite{kool2018attention}, and recommendation systems \cite{chen2018neural,deudon2018learning}.
In many works for combinatorial optimization problems, they do not consider the future state after selecting a combination of $K$ items and some other commands.
\cite{chen2018neural} suggests similar cascaded Q-networks without efficient weight sharing which is crucial in handling large dimensional items. 
\cite{usunier2016episodic} suggests a centralized MARL algorithm where the agent randomly selects an item first and then considers the command. 
Independent Deep Q-network (IDQN) \cite{tampuu2017multiagent} is an MARL algorithm where each item independently chooses its command using its Q-network.
To summarize, our contribution is to extend and integrate those combinatorial optimization problems successfully and to provide a scalable RL algorithm using weight shared Q-networks. 

\myparagraph{Parameter Sharing on Neural Networks and Analysis}
Parameter shared neural networks have been studied on various structured data domains such as graphs \cite{kipf2016semi} and sets \cite{qi2017pointnet}.
These networks do not only save significant memory and computational cost but also perform usually better than non-parameter shared networks.
For the case of set-structured data,
there are two major categories: equivariant \cite{ravanbakhsh2016deep,hartod2018sets} and invariant networks \cite{qi2017pointnet,zaheer2017deep,maron2019universality}. 
In this paper, we develop a parameter shared network (I-sharing) which contains both permutation equivariant and invariant properties.
Empirical successes of parameter sharing have led many works to delve into its mathematical properties.
\cite{qi2017pointnet,zaheer2017deep,maron2019universality} show the universality of invariant networks for various symmetries.
As for equivariant networks, a relatively small number of works analyze their performance.
\cite{ravanbakhsh2017equivariance,zaheer2017deep,hartod2018sets} find necessary and sufficient conditions of equivariant linear layers.
\cite{yarotsky2018universal} designs a universal equivariant network based on polynomial layers.
However, their polynomial layers are different from widely used linear layers.
In our paper, we prove two theorems which mathematically guarantee the performance of our permutation equi-invariant networks in different ways. 
Both theorems can be applied to other similar related works.

\section{Preliminary}
 \label{sec:model}

\subsection{Iterative Select-MDP (IS-MDP)}
\label{sec:is-mdp}

\begin{figure}
  \centering
  \begin{minipage}{.48\textwidth}
    \captionsetup[subfloat]{}
    \subfloat[Select-MDP (S-MDP)]{\includegraphics[width=0.99\linewidth]{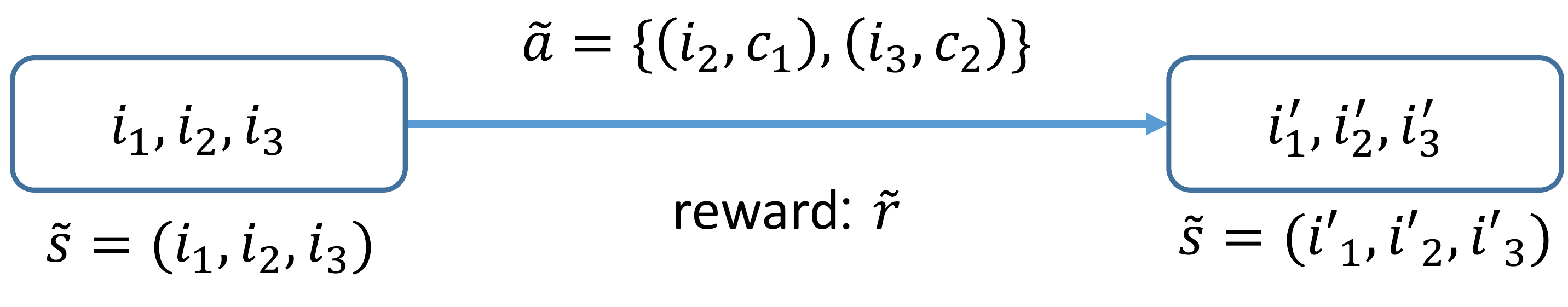}}
  \end{minipage}
  \begin{minipage}{.48\textwidth}
    \subfloat[Iterative-Select MDP (IS-MDP)]{\includegraphics[width=0.99\linewidth]{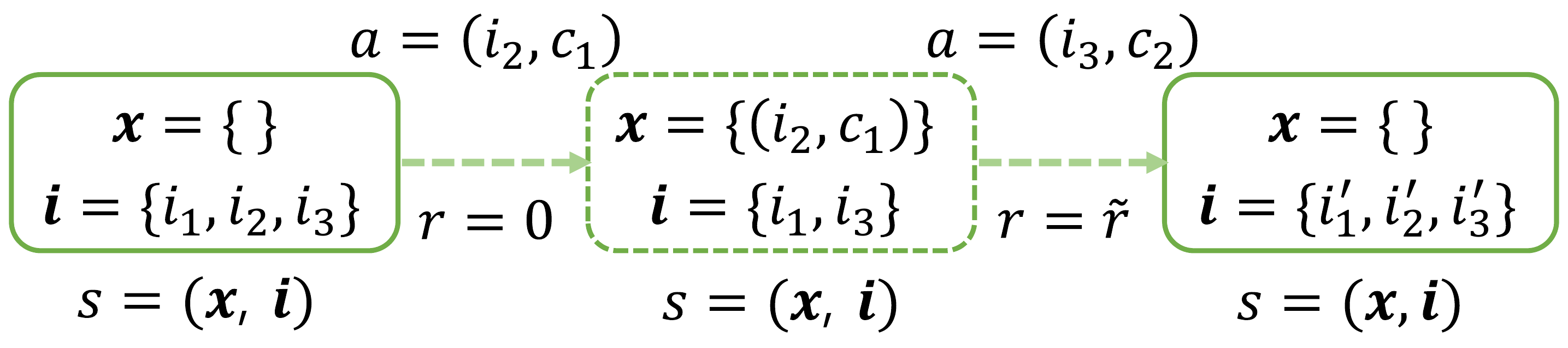}}
  \end{minipage}
  \caption[cross]{
    Example of an S-MDP and its equivalent IS-MDP for $N=3$ and $K=2$.
  }
  \label{fig:MDPs}
\end{figure}

Given an S-MDP, we formally describe an IS-MDP as a tuple $\cM=\langle \tcS, \tcA, \tcP, \tcR, \tgam \rangle$ that makes a selection of $K$ items and corresponding commands in an S-MDP through $K$ consecutive selections. 
Fig.~\ref{fig:MDPs} shows an example of the conversion from an S-MDP to its equivalent IS-MDP. 
In IS-MDP, given a tuple of the $N$-item information $(i_1,\ldots,i_N)$, with $i_n \in \cI$ being the information of the item $n$, the agent selects one item $i_n$ and assigns a command $c \in \cC$ at every `phase' $k$ for $0 \leq k < K$. After $K$ phases, it forms a joint selection of $K$ items and commands, and a probabilistic transition of the $N$-item information and the associated reward are given.

To elaborate, at each phase $k$, the agent observes a state $s = ((x_1,\ldots,x_{k}),(i_{1},\ldots,i_{N-k})) \in \cS_k$ which consists of a set of $k$ pairs of information and command which are selected
 in prior phases, denoted as $\bx = (x_1,\ldots,x_{k})$, with $x_k \in \cI \times \cC$  being a pair selected in the $k$th phase, and a tuple of information of the unselected items up to  phase $k$, denoted as $\bi = (i_{1},\ldots,i_{N-k})$. 
From the observation $s \in \cS_k$ at phase $k$, the agent selects an item $n$ among the $N-k$ unselected items and assigns a command $c$, i.e.,
 a feasible action space for state $s$ is given by $\cA(s):=\{ (n,c) \, \vert \, n \in \{1,\ldots,N-k\}, \, c \in \cC\}$, where $(n,c)$ represents a selection $(i_{n},c)$. As a result, the state and action spaces of an IS-MDP are given by $\cS = \bigcup_{0 \leq k < K} \cS_k$ and $\cA = \bigcup_{s \in \cS}\cA(s)$, respectively. 
We note that any state $\tilde{s} = (i_1,\ldots,i_N)$ in an S-MDP belongs to $\cS_0$, i.e., the $0$th phase. In an IS-MDP, action $a=(n,c) \in \cA(s)$ for state $s = (\bx,\bi) \in \cS_k$ results in the next state $s' = (\bx+(i_{n},c), \bi - i_{n}) \in \cS_{k+1}$\footnote{We use $+,-$ as  $\bx + x:=(x_{1}, \cdots, x_{k}, x)$ and $\bi - i_{n}:= (i_{k}, \cdots, i_{n-1}, i_{n+1}, \cdots, i_{N})$.} and a reward $r$, 
\begin{align} \label{eq:EQUIGreedyTran}
    \begin{split}
        & k < K-1, \quad \cP(s',0 \mid s,a) \equiv 1, \\
        & k = K-1, \quad \cP(s',\tilde{r} \mid s,a) \equiv \tilde{\cP}(\tilde{s}',\tilde{r} \mid \tilde{s},\tilde{a}).
    \end{split}
\end{align}
Recall $\tilde{\cP}$ is the transition probability of S-MDP.
The decomposition of joint action in S-MDPs (i.e., selecting $K$ items at once) into $K$ consecutive selections in IS-MDPs has equivalence in terms of the optimal policy \cite{maes2009structured}. 
Important advantage from the decomposition is that IS-MDPs have action space $\cA$ of size $NC$ while the action space of S-MDPs is $\binom{N}{K} C^K$. 

\subsection{Deep Q-network (DQN)} \label{sec:dqn}

We provide a background of the DQN \cite{mnih2015human}, one of the standard deep RL algorithms, whose key ideas such as the target network and replay buffer will also be used in our proposed method. 
The goal of RL is to learn an optimal policy 
$\pi^{\star}(a|s): \cS \times \cA \mapsto [0,1]$ that maximizes the expected discounted return. 
We denote the optimal action-value functions (Q-function) under the optimal policy $\pi^{\star}$ by $Q^{\star}(s,a)$.
The deep Q-network (DQN) parameterizes and approximates the optimal Q-function $Q^{\star}(s,a)$ using the so-called Q-network $Q(s,a; \omega)$, i.e., a deep neural network with a weight parameter vector $\omega$. In DQN, the parameter $\omega$ is learned by sampling minibatches of experience $(s,a,r,s')$ 
from the replay buffer and using the following loss function:
\begin{align} \label{eq:loss_general}
  l(\omega) = \Big(\tQ(\ts,\ta\,;\omega) - \large(r+\tgam \max_{\ta ' \in \cA(s')}  \tQ(\ts ',\ta'\,;\omega') \large)  \Big)^{2}
\end{align}
where $\omega'$ is the target parameter which follows the main parameter $\omega$ slowly. 
It is common to approximate $\tQ(\ts;\omega):\tcS \mapsto \Real^{|\tcA(s)|}$ rather than $\tQ(\ts, \ta;\omega)$ using a neural network so that all action values can be easily computed at once.

\section{Methods}
\label{sec:Methods}

In this section, we present a symmetric property of IS-MDP,
which is referred to as {\em Equi-Invariance} (EI), and propose an
efficient RL algorithm to solve IS-MDP by constructing $K$ cascaded
Q-networks with two-levels of parameter sharing.

\subsection{IS-MDP: Equi-Invariance}
As mentioned in Sec.~\ref{sec:is-mdp}, a state $s = (\bx,\bi)$ at
phase $k$ includes two sets $\bx$ and $\bi$ of observations, so that we
have some permutation properties related to the ordering of elements
in each set, i.e., for all $s,s' \in \cS$, $a\in \cA$, and $r\in \cR$, 
\begin{align} \label{eq:sym_tran}  \cP(s',r\mid s,a) \equiv \cP(s',r \mid \sigma_{s}(s),\sigma_{i}(a)). \end{align}
We denote $\sigma_{s} =(\sigma_{x}, \sigma_{i}) \in \bS_{k} \times \bS_{N-k}$ as a permutation of a state $s$ at phase $k$, which is defined as
\begin{align} \label{eq:sigma_s} 
  \sigma_{s}(s) := (\sigma_{x}(\bx), \sigma_{i}(\bi)),
\end{align}
where $\bS_k$ is a group of permutations of a set with $k$
elements.
From \eqref{eq:sym_tran}, we can easily induce that if the action $a=(n,c) \in \cA(s)$ is the optimal action for $s$, then for state $\sigma_{s}(s)$, an optimal policy 
 should know that an permuted action $\sigma_{i}(a):=(\sigma_{i}(n),c)$ is also optimal.
As a result, we have $\forall s \in \cS, \forall a \in \cA(s)$,
\begin{align} \label{eq:Q(s,a)Invariant}
Q^{\star}(s,a) = Q^{\star}(\sigma_{s}(s),\sigma_{i}(a)).
\end{align}

Focusing on Q-value function
$Q^\star(s) = [Q^\star(s,a)]_{a \in \cA(s)}$, as discussed in
Sec.~\ref{sec:dqn}, a permutation $\sigma_s=(\sigma_x,\sigma_i)$ of a state $s$ permutes the output of the function $Q^\star(s)$ according
to the permutation $\sigma_i$. In other words, a state $s$ and the permutation thereof, $\sigma_s(s)$, have {\em
  equi-invariant} optimal Q-value function $Q^\star(s)$. This is stated in the following proposition which is a rewritten form of \eqref{eq:Q(s,a)Invariant}.
\begin{prop} [Equi-Invariance of IS-MDP]
  \label{prop:ei_optimal}
  In IS-MDP, the optimal Q-function $Q^\star(s)$ of any state
  $s=(\bx,\bi) \in \cS$ is invariant to the permutation of a set $\bx$
  and equivariant to the permutation of a set $\bi$, i.e. for any
  permutation $\sigma_s=(\sigma_x,\sigma_i)$,
\begin{align} \label{eq:equi-inv property}
  Q^\star(\sigma_{s}(s)) = \sigma_{i}(Q^{\star}(s)).
\end{align}
\end{prop}

As we will discuss later, this EI property in \eqref{eq:equi-inv property} plays a critical
role in reducing state and action spaces by considering $(s,a)$ pairs
and permutations thereof to be the same. 
We follow the idea in \cite{zinkevich2001symmetry} to prove  Proposition~\ref{prop:ei_optimal}.

\subsection{Iterative Select Q-learning (ISQ)}
\label{sec:isq}

\paragraph{Cascaded Deep Q-networks} As mentioned in
Sec.~\ref{sec:is-mdp}, the dimensions of state and action spaces differ over phases. 
In particular, as the phase $k$ progresses, the set $\bx$ of the
state increases while the set $\bi$ and the action space $\cA(s)$
decrease. Recall that the action space of state $s \in \cS_k$ is $\cA(s) =
\{(n,c) \mid n \in \{1,\ldots,N-k\}, c \in \cC \}$. Then, Q-value
function at each phase $k$, denoted as
$Q_k(s)=[Q(s,a)]_{a \in \cA(s)}$ for $s \in \cS_k$, is characterized
by a mapping from
a state space $\cS_k$ to $\mathbb{R}^{(N-k)\times C}$,
where the $(n,c)$-th output element corresponds to the value $Q(s,a)$
of $a=(n,c) \in \cA(s)$.

To solve IS-MDP using a DQN-based scheme, we
construct $K$ deep Q-networks that are cascaded, where the $k$th Q-network, denoted
as $Q_k(s;\omega_k)$, approximates the Q-value function $Q_k(s)$ with
a learnable parameter vector $\omega_k$. We denote by
$\bomega = \{\omega_k \}_{0 \leq k < K}$ and
$\bomega' = \{ \omega_k' \}_{0 \leq k < K}$ the collections of the
main and target weight vectors for all $K$-cascaded Q-networks,
respectively. With these $K$-cascaded Q-networks,
DQN-based scheme can be applied to each Q-network $Q_k(s;\omega_k)$
for $0 \leq k < K$ using the associated loss function as in
\eqref{eq:loss_general} with $\omega = \omega_k$ and
$\omega' = \omega_{k+1}'$ (since $s' \in \cS_{k+1}$), which we name {\em Iterative Select Q-learning (ISQ)}.

Clearly, a naive ISQ algorithm would have training challenges due to the large-scale of $N$ and $K$ since (i) number of parameters in each network $\omega_k$ increases as $N$ increases and 
(ii) size of the parameter set $\bomega$ also increases as $K$ increases.
To overcome these, we propose parameter sharing ideas which are described next.  

\begin{figure*}[ht]
  \centering
  \includegraphics[width=.95\linewidth]{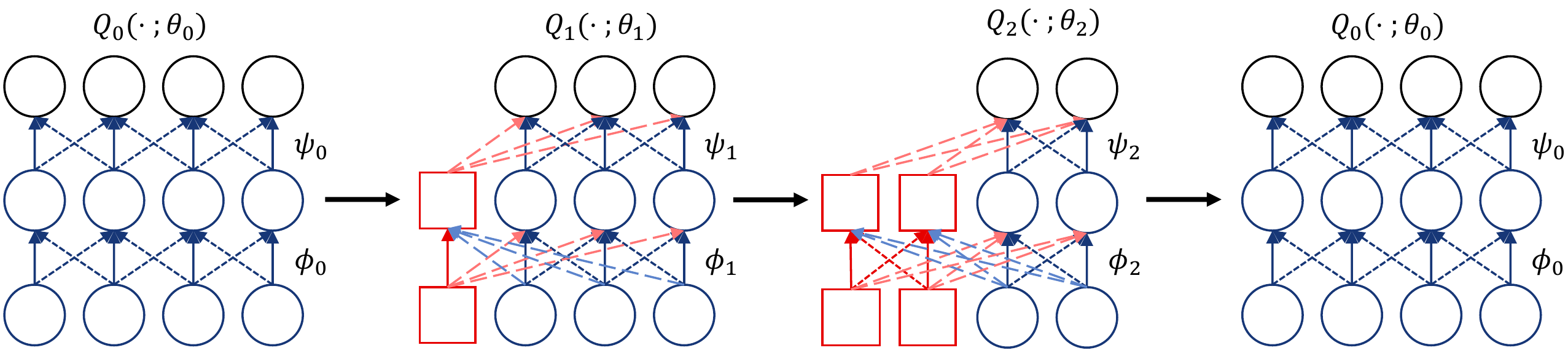}
  \caption[cross]{A simple example of the parameter-shared Q-networks $Q_{k}(\cdot \, ; \theta_{k})$ when $K=3, N=4, |\mathcal{C}|=1$.
  Red and blue colored nodes represent the nodes equivariant to the selected items $\bx$ and the unselected items $\bi$ respectively. Each black node represents the Q value for selecting the corresponding (item, command) pair.
  }
  \label{fig:Perphase_networks}
\end{figure*}

\myparagraph{Intra Parameter Sharing (I-sharing)}
To overcome the parameter explosion for large $N$ in each Q-network,
we propose a parameter sharing scheme, called 
{\em intra parameter sharing} (I-sharing). Focusing on
the $k$th Q-network without loss of generality, the Q-network with I-sharing has a reduced
parameter vector $\theta_k$\footnote{To distinguish the parameters of Q-networks with and without I-sharing, 
we use notations $\theta_k$ and $\omega_k$ for each case, respectively.}, yet it satisfies 
the EI property in \eqref{eq:equi-inv property}, as discussed shortly. 

The Q-network with I-sharing $Q_k(\cdot;\theta_k)$ is a multi-layered
neural network constructed by stacking two types of parameter-shared
layers: $\phi_k$ and $\psi_k$. As illustrated in
Fig.~\ref{fig:Perphase_networks}, where the same colored and dashed weights are
tied together,
 the layer $\phi_k$ is designed
to preserve an {\em equivariance} of the permutation
$\sigma_s = (\sigma_x, \sigma_i) \in \bS_k \times \bS_{N-k}$, while
the layer $\psi_k$ is designed to satisfy {\em invariance} of $\sigma_x$ as
well as {\em equivariance} of $\sigma_i$, i.e.,
\begin{align*}
  \phi_k(\sigma_s(\bx,\bi)) &= \sigma_s(\phi_k(\bx,\bi)), \cr
  \psi_k(\sigma_s(\bx,\bi)) &= \sigma_i(\psi_k(\bx,\bi)).
\end{align*}
Then, we construct the Q-network with I-sharing $Q_k(\cdot;\theta_k)$ 
by first stacking multiple layers of $\phi_k$ followed by a
single layer of $\psi_k$ as
\begin{align*}
Q_k(s; \theta_k) := \psi_k \circ \phi_k \circ \cdots \circ \phi_k(s),
\end{align*}
where $\theta_k$ is properly set to have tied values. Since
composition of the permutation equivariant/invariant layers preserves the
permutation properties, we obtain the following EI property
\begin{align*}
Q_k(\sigma_s(\bx,\bi); \theta_k) = \sigma_i(Q_k(\bx,\bi; \theta_k)).
\end{align*}
ISQ algorithm with I-sharing, termed ISQ-I, achieves a significant reduction of the number of parameters from $|\bomega| = O(N^2K)$ to
$|\btheta| = O(K)$, where $\btheta = \{ \theta_k\}_{0
  \leq k < K}$ is the collection of the parameters. 
  We refer the readers to our technical report\footnote{https://github.com/selectmdp} 
  for a more mathematical description. 

\myparagraph{Unified Parameter Sharing (U-sharing)}
We propose an another-level of weight sharing method for ISQ, called {\em unified parameter sharing} (U-sharing). 
We observe that each I-shared Q-network $Q_{k}(\cdot \, ; \, \theta_{k})$ has a fixed number of parameters regardless of phase $k$.
This is well described in Fig.~\ref{fig:Perphase_networks}, where the number of different edges are the same in $Q_{1}$ and $Q_{2}$.
From this observation, we additionally share $\theta_{k}$ among the different Q-networks $Q_{k}$, i.e. $\theta_{0} = \cdots = \theta_{K-1}$.
U-sharing enables the reduction of the number of weights from 
$O(K)$ for $\btheta$ to $O(1)$ for $\theta_{0}= \cdots = \theta_{K-1}$. 
Our intuition for U-sharing is that
since the order of the selected items does not affect the transition of S-MDP, 
an item which must be selected during $K$ phases has the same Q-values in every phase.\footnote{Note that, we set the discount factor $\gamma=0$ during except the final phase $K-1$.}
This implies that the weight vectors $\theta_{k}$ may also have similar values.
However, too aggressive sharing such as sharing all the weights may experience significantly reduced expressive power. 
\myparagraph{Progressive Parameter Sharing (P-sharing)}
To take the advantages of both I- and U-sharing, we propose a combined method called \textit{progressive parameter sharing} (P-sharing). In P-sharing, we start with a single parameter set (as in U-sharing) and then progressively double  the number of sets until it reaches $K$ (the same as I-sharing).
The Q-networks with nearby phases ($Q_{k}$ and $Q_{k+1}$) tend to share a parameter set longer as visualized in Fig.~\ref{fig:Progressive_learning}, which we believe is because they have a similar criterion.
In the early unstable stage of the learning,
the Q-networks are trained sample-efficiently as they exploit the advantages of U-sharing.
As the training continues, the Q-networks are able to be trained more elaborately, with more accurate expressive power, by increasing the number of parameter sets. In P-sharing, the number of the total weight parameters ranges from $O(1)$ to $O(K)$ 
during training.

\begin{figure}[t!]
  \centering
  \includegraphics[width=.99\linewidth]{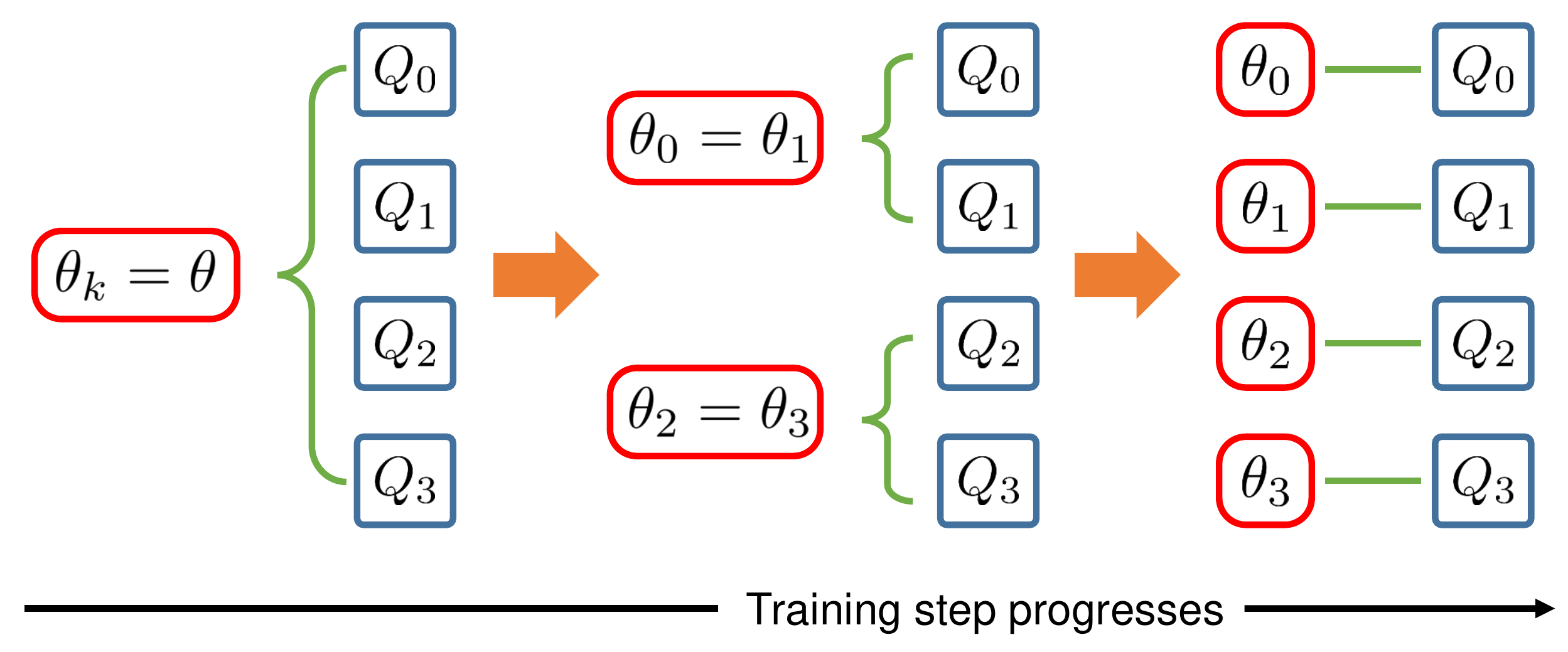}
  \caption[cross]{Illustration of P-sharing for $K=4$. In the beginning, all Q-networks share the same weights. As the training progresses, we double the number of parameter sets until each Q-network $Q_{k}$ is trained with its own parameter vectors $\theta_{k}$.
  }
  \label{fig:Progressive_learning}
\end{figure}

\section{Intra-Sharing: Relativ Local Optimality and Universal Approximation}
One may naturally raise the question of whether the I-shared Q-network $Q_{k}(s \, ; \theta_{k}): \cS_{k} \rightarrow \Real^{\vert \cA(s) \vert}$
has enough expressive power to represent the optimal Q-function
$Q_{k}^{\star}(s): \cS_{k} \rightarrow \Real^{\vert \cA(s) \vert}$
of the IS-MDP despite the large reduction in the number of the parameters from $O(N^{2}K)$ to $O(K)$.
In this section, we present two theorems that show $Q_{k}(s \, ;\theta_{k})$ has enough expressive power to approximate $Q_{k}^{\star}(s)$ with the EI property in \eqref{eq:equi-inv property}. Theorem~\ref{thm:local optimal} states how I-sharing affects local optimality and Theorem~\ref{thm:universal} states whether the network still satisfies the universal approximation even with the equi-invariance property. Due to space constraint, 
we present the proof of the theorems in the technical report. 
We comment that both theorems can be directly applied to other similar weight shared neural networks, e.g., \cite{qi2017pointnet,zaheer2017deep,ravanbakhsh2017equivariance}.
For presentational convenience, we denote
$Q_{k}^{\star}(s)$ as $Q^{\star}(s)$,
$Q_{k}(s \, ;\omega_{k})$ as $Q_{\omega}(s)$,
and $Q_{k}(s \, ;\theta_{k})$ as $Q_{\theta}(s).$




\myparagraph{Relative Local Optimality} We compare the expressive power of I-shared Q-network $Q_{\theta}$ and vanilla Q-network $Q_{\omega}$ of the same structure when approximating a function $Q^{\star}$ satisfies the EI property.
Let $\Theta$ and $\Omega$ denote weight vector spaces for $Q_{\theta}$ and $Q_{\omega}$, respectively.
Since both $Q_{\omega}$ and $Q_{\theta}$ have the same network sructure,
we can define a projection mapping $\omega: \Theta \rightarrow \Omega$ such that $Q_{\omega(\theta)} \equiv Q_{\theta}$ for any $\theta$.
Now, we introduce a loss surface function $l_{\Omega}(\omega)$ of the weight parameter vector $\omega$:
$$   l_{\Omega}(\omega) := \sum_{s \in B} |  Q_{\omega}(s) - Q^{\star}(s)|^2, $$
where $B \subset \cS_{k}$ is a batch of state samples at phase $k$ and $Q^{\star}(s)$ implies the true Q-values to be approximated.
Note that this loss surface $l_{\Omega}$ is different from the loss function of DQN in \eqref{eq:loss_general}.
However, from the EI property in $Q^{\star}(s)$, we can augment additional true state samples and the true Q-values by using equivalent states for all $\sigma_{s} \in \bS_{k}\times \bS_{N-k}$,
$$ L_{\Omega}(\omega) := \sum_{\sigma_{s} \in \bS_k \times \bS_{N-k} } \left(  \sum_{s \in B} |  Q_{\omega}(\sigma_{s}(s)) - Q^{\star}(\sigma_{s}(s))|^2 \right). $$
We denote the loss surface $L_{\Theta}(\theta):=L_{\Omega}(\omega(\theta))$ in the weight shared parameter space $\Theta$.
\begin{theorem}[Relative Local Optimality] \label{thm:local optimal}
  If $\theta^{\star} \in \Theta$ is a local optimal parameter vector of the loss surface $L_{\Theta}(\theta)$, then the projected parameter $\omega(\theta^{\star}) \in \Omega$ is also the local optimal point of $L_{\Omega}(\omega)$.
\end{theorem}

It is notoriously hard to find a local optimal point by using gradient descent methods because of many saddle points in high dimensional deep neural networks \cite{dauphin2014identifying}.
However, we are able to efficiently seek for a local optimal parameter $\theta^{\star}$ on the smaller dimensional space $\Theta$, rather than exploring $\Omega$.
The quality of the searched local optimal parameters $\omega(\theta^{\star})$ is reported to be reasonable 
that most of the local optimal parameters give nearly optimal performance in high dimensional neural networks \cite{dauphin2014identifying,kawaguchi2016deep,laurent2018deep} 
To summarize, Theorem~\ref{thm:local optimal} implies that $Q_{\theta}$ has similar expressive power to $Q_{\omega}$ if both have the same architecture.


\myparagraph{Universal Approximation}
We now present a result related to the universality of $Q_{\theta}(s)$ when it approximates $Q^{\star}(s)$.
\begin{theorem}[Universal Approximation] \label{thm:universal}
Let  $ Q^{\star}:\cS_{k} \rightarrow \Real^{(N-k)\times C}$ satisfies EI property.
If the domain spaces  $\cI$ and $\cC$ are compact, for any $\epsilon>0$, there exists a $4$-layered I-shared neural  network $Q_{\theta}: \cS_{k} \rightarrow \Real^{(N-k)\times C}$ with a finite number of neurons, which satisfies $$ \forall  s \in \cS_{k}, \quad  \vert Q^{\star}(s) - Q_{\theta}(s) \vert < \epsilon. $$
\end{theorem}

Both Theorems~\ref{thm:local optimal} and \ref{thm:universal} represent the expressive power of the I-shared neural network for approximating an equi-invariant function. However, they differ in the sense that 
Theorem~\ref{thm:local optimal} directly compares the expressive power of the I-shared network to the network without parameter sharing, whereas Theorem~\ref{thm:universal} states the potential power of the I-shared network that 
any function $f$ with EI property allows good approximation as the number of nodes in the hidden layers sufficiently increase.

\section{Simulations} \label{sec:eval}

%
%
\subsection{Environments and Tested Algorithms}

\begin{figure*}[ht]
  \centering
  \captionsetup[subfloat]{captionskip=-0.03cm}
  \subfloat[Single selection ($K=1$)]{\label{fig:k1n52050}
  \includegraphics[width=0.33\textwidth]{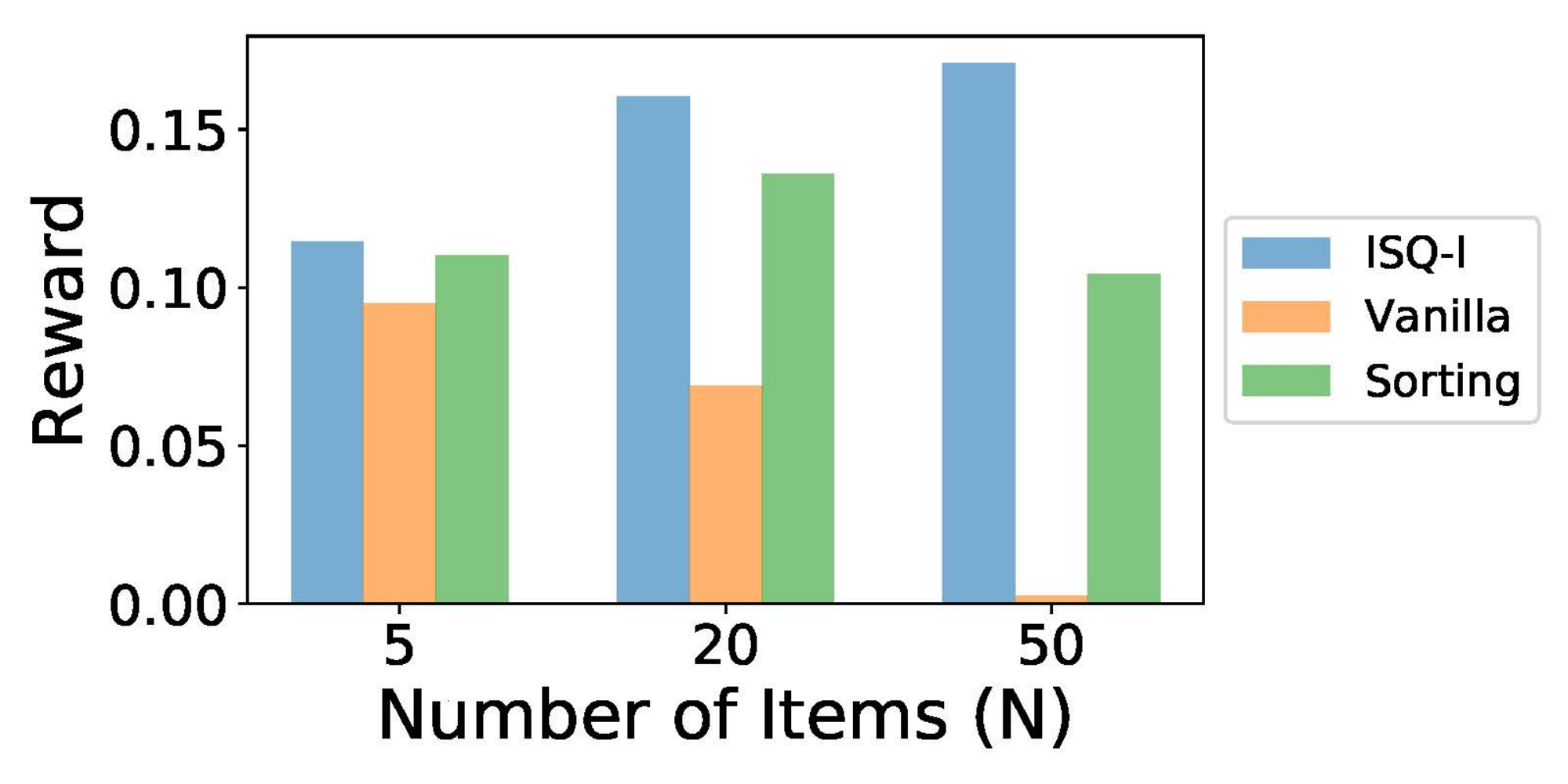}}
  \subfloat[$N=50,K=6$]{\label{fig:k6n50}\includegraphics[width=0.33\textwidth]{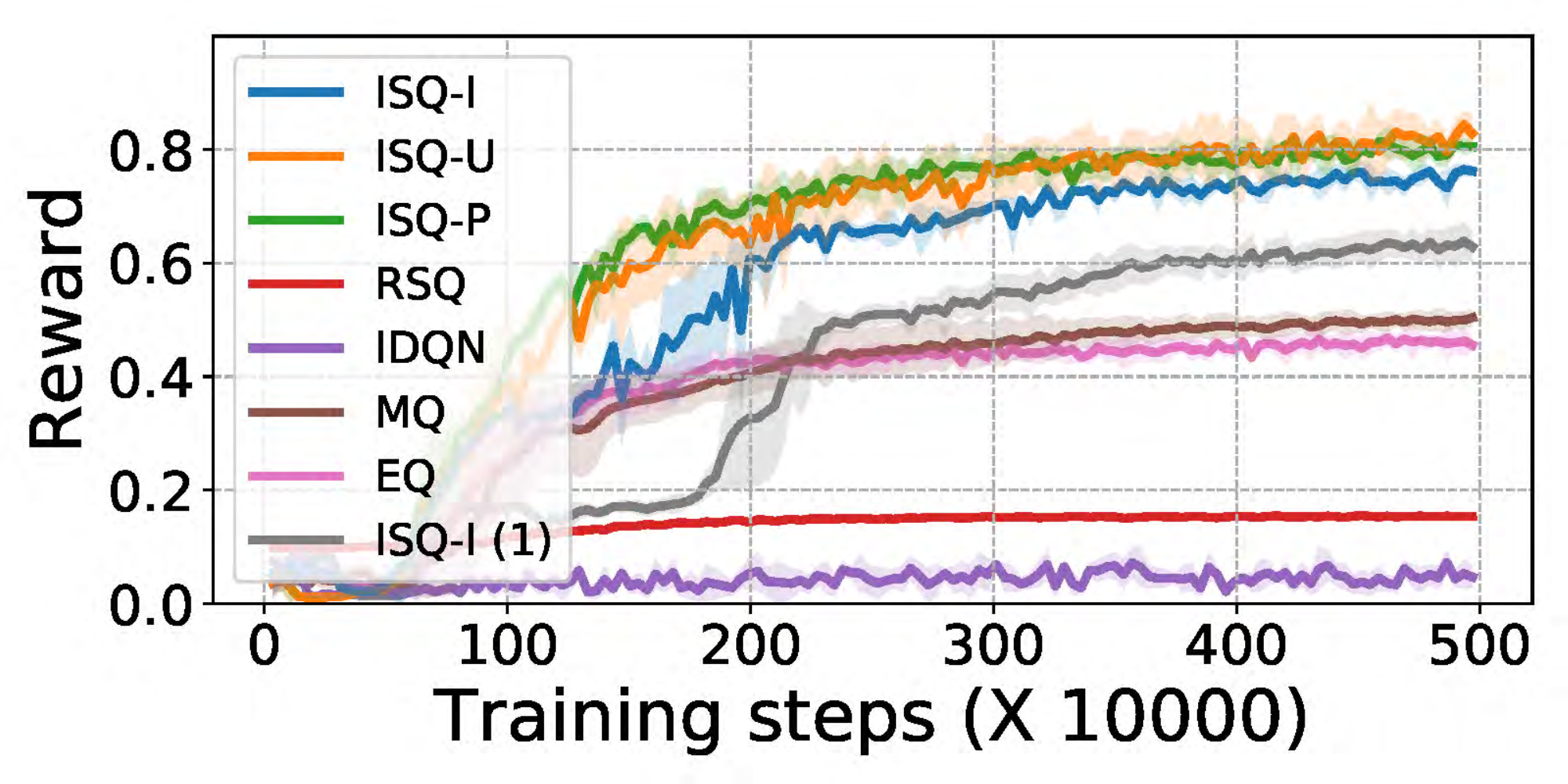}}
  \subfloat[$N=200,K=6$]{\label{fig:k6n200}\includegraphics[width=0.33\textwidth]{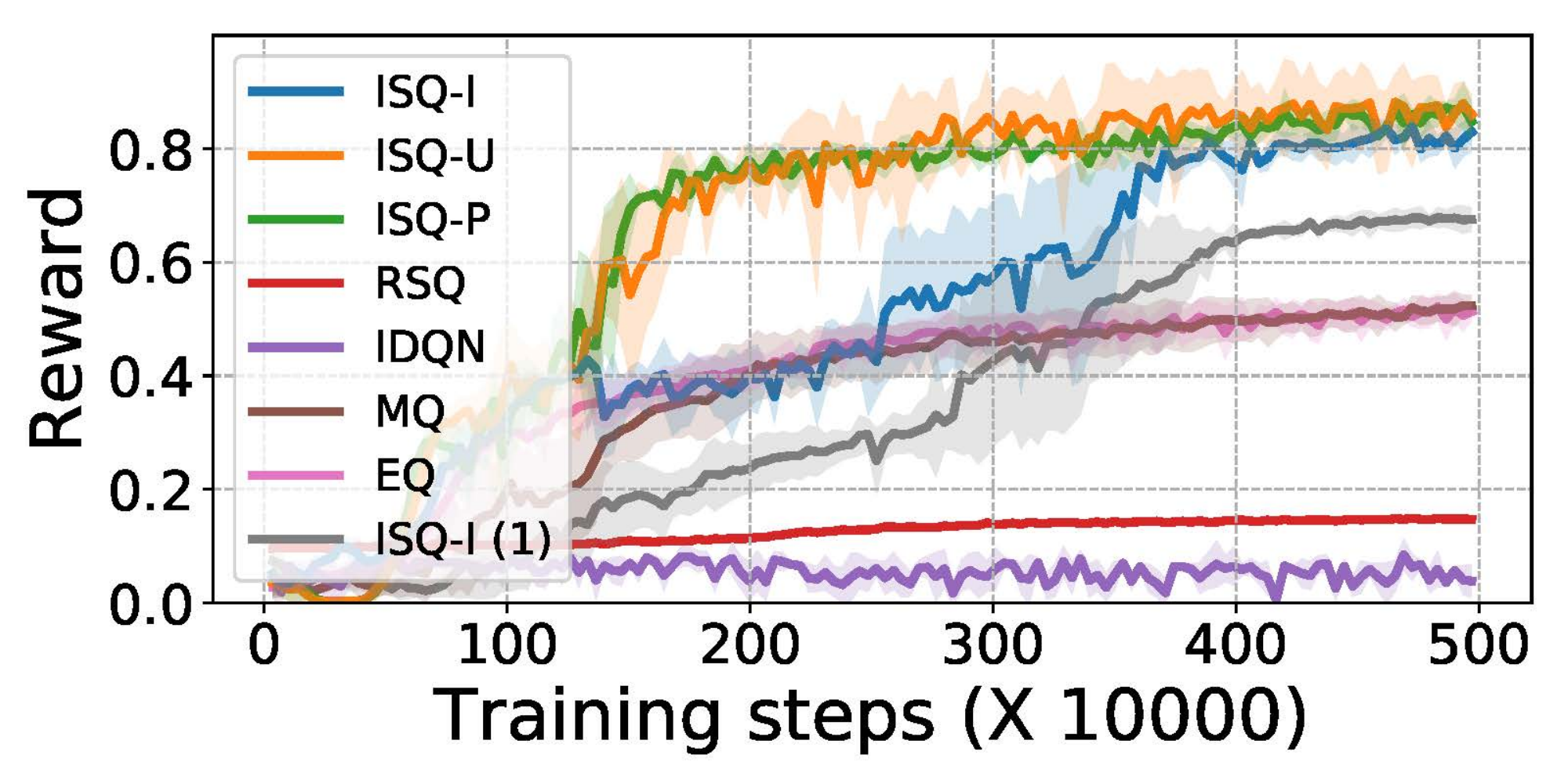}}
  \caption[cross]{
  Performances for CS tasks.   (a): final performances of the methods for single selection with $N=5,20,50$. (b) and (c): 
  learning curves for $K=6,U=0$  with $N=50,200$.
  ISQ-I (1) corresponds to the ISQ-I with a single command `stay'.}
  \label{fig:K6}
\end{figure*}

\begin{figure*}[ht]
  \centering
  \subfloat[$N=10, K=4$ \label{fig:K4N10} ]{\includegraphics[width=0.33\textwidth]{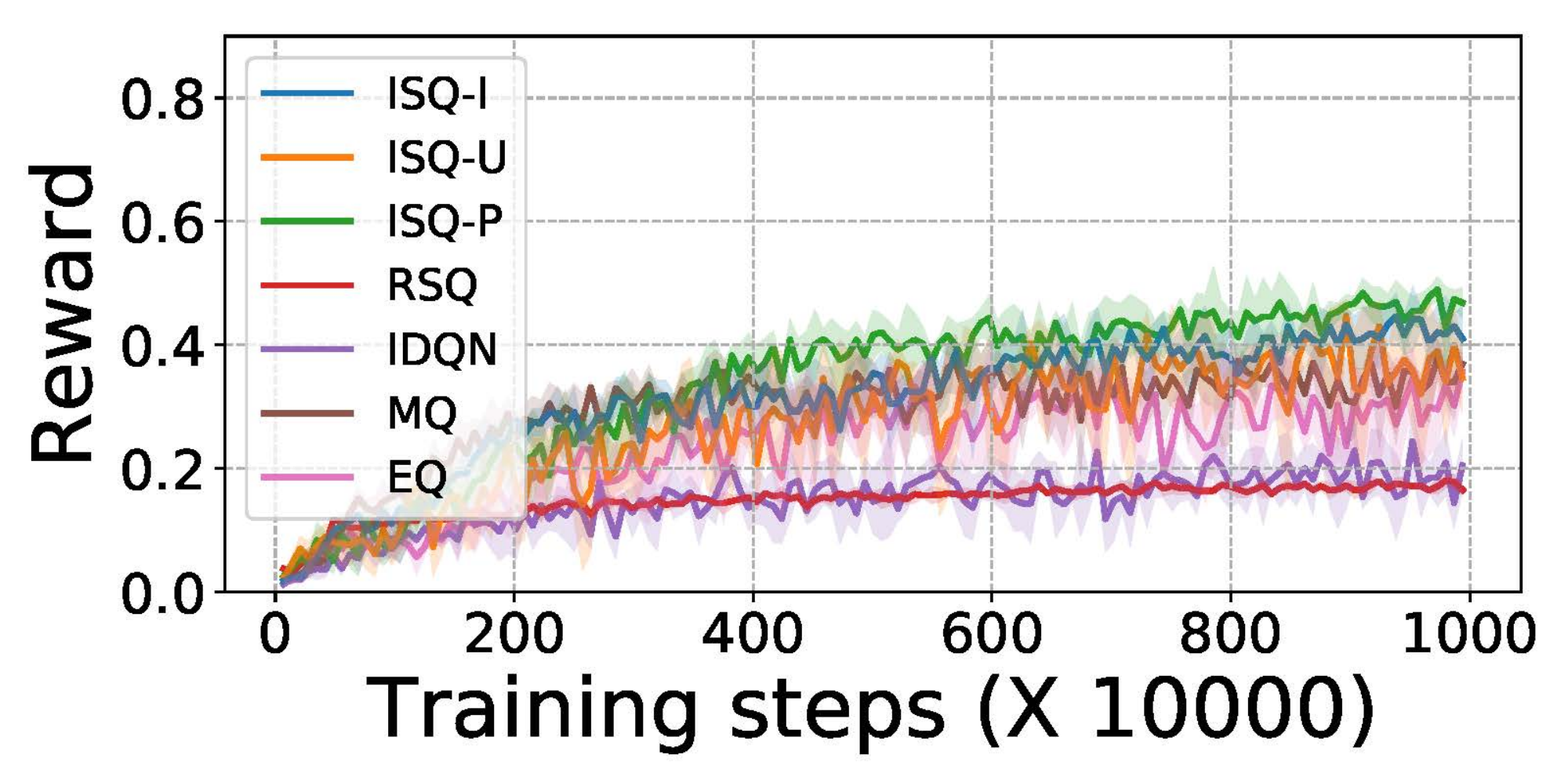}}
  \subfloat[$N=10, K=7$\label{fig:K7N10}]{\includegraphics[width=0.33\textwidth]{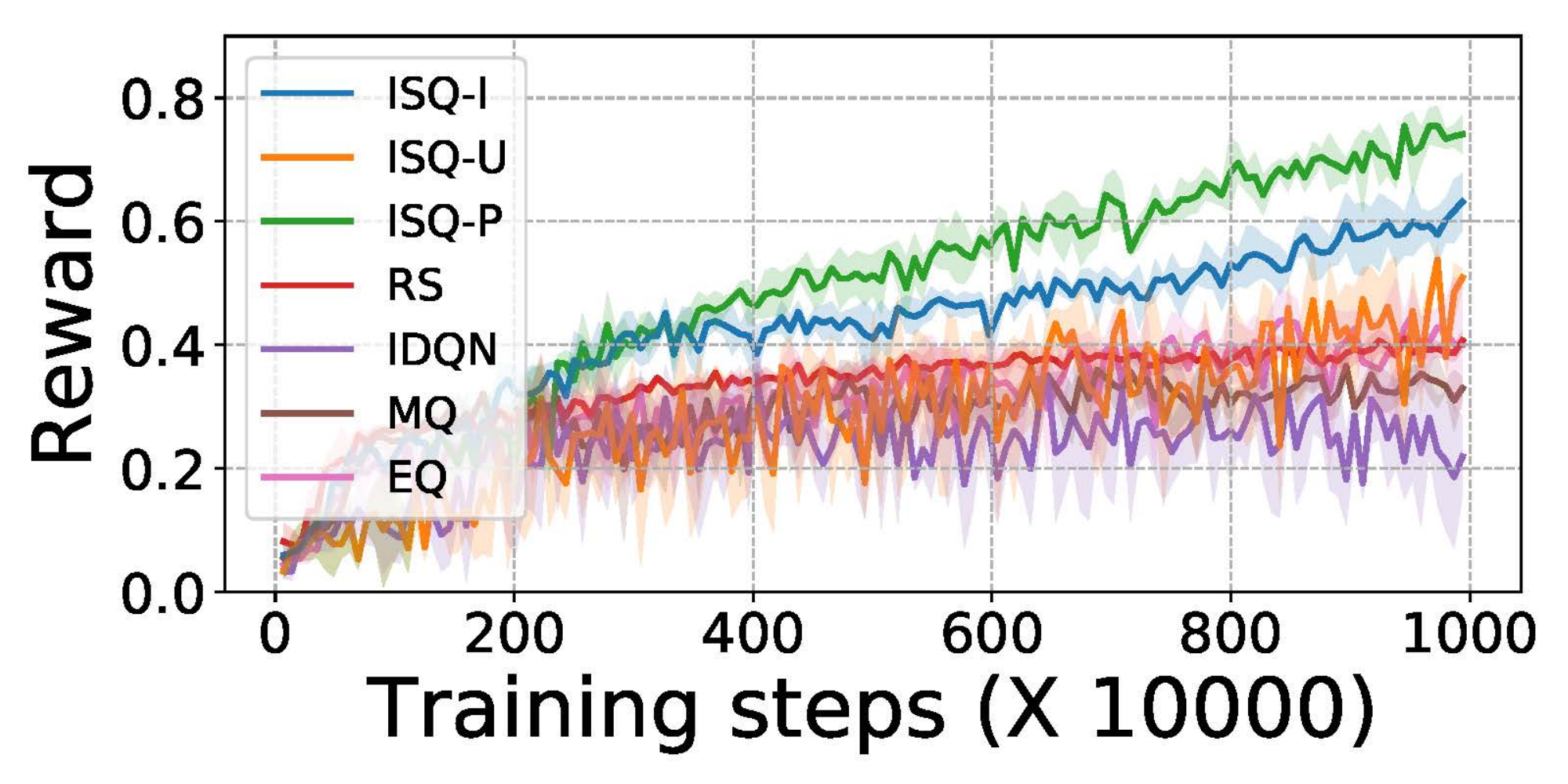}}
  \subfloat[$N=10, K=10$\label{fig:K10N10}]{\includegraphics[width=0.33\textwidth]{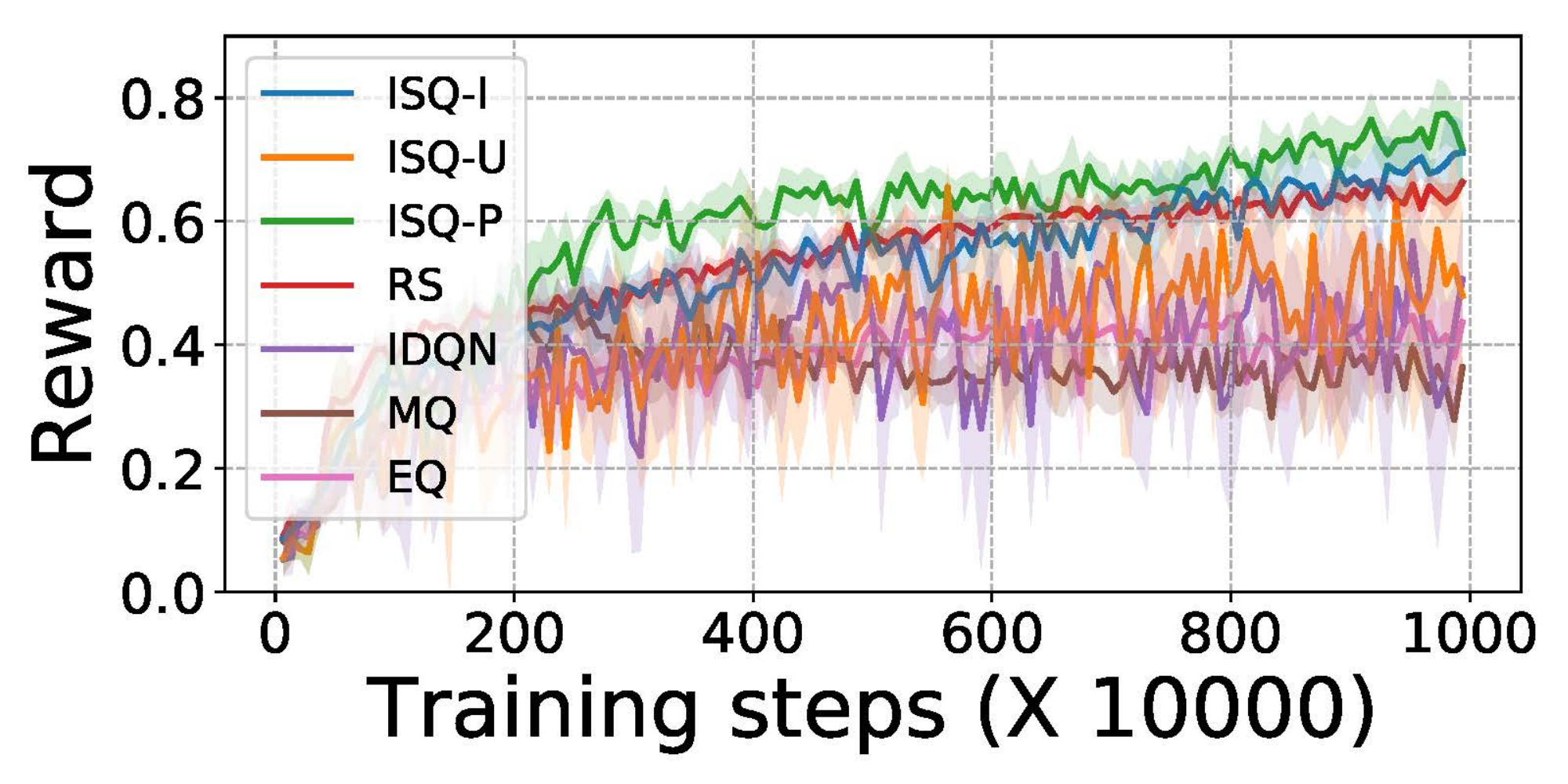}}
  \caption[cross]{Learning curves for the PP task with $10$ predators and $4$ preys.
  Each episode consists of 175 steps. 
  }
  \label{fig:predator-prey}
\end{figure*}

\paragraph{Circle Selection (CS)} \label{sec:exp-CS} 
In Circle Selection (CS) task, there are $N$ selectable and $U$ unselectable circles, where each circle is randomly moving and its radius increases with random noise. The agent observes positions and radius values of all the circles as a state, selects $K$ circles among $N$ selectable ones, and chooses $1$ out of the $5$ commands: moves {\em up, down, left, right}, or {\em stay}. Then, the agent receives a negative or zero reward if the selected circles overlap with unselectable or other selected circles, respectively; otherwise, it can receive a positive reward. The amount of reward is related to a summation of the selected circles' area. All selected circles and any overlapping unselectable circle are replaced by new circles, which are initialized at random locations with small initial radius. Therefore, the agent needs to avoid the overlaps by carefully choosing circles and their commands to move. 

\myparagraph{Selective Predator-Prey (PP)}
In this task, multiple predators capture randomly moving preys. The agent observes the positions of all the predators and preys, selects $K$ predators, and assigns the commands as in the CS task. Only selected predators can move according to the assigned command and capture the preys. The number of preys caught by the predators is given as a reward, where a prey is caught if and only if more than two predators catch the prey simultaneously.

\myparagraph{Tested Algorithms and Setup}
We compare the three variants of ISQ: ISQ-I, ISQ-U, ISQ-P with three DQN-based schemes: (i) a vanilla DQN \cite{mnih2015human}, (ii) a sorting DQN that reduces the state space by sorting the order of items based on a pre-defined rule, and (iii) a myopic DQN which learns to maximize the instantaneous reward for the current step, but follows all other ideas of ISQ. 
We also consider three other baselines motivated by value-based MARL algorithms in \cite{tampuu2017multiagent,usunier2016episodic,chen2018neural}: 
Independent DQN (IDQN), Random-Select DQN (RSQ), and Element-wise DQN (EQ). In IDQN, each item observes the whole state and has its own Q-function with action space equals to $\cC$. In RSQ, the agent randomly selects items first and chooses commands from their Q-functions. EQ uses only local information to calculate each Q-value. We evaluate the models by averaging rewards with $20$ independent episodes. The shaded area in each plot indicates $95\%$ confidence intervals in $4$ different trials, where all the details of the hyperparameters are provided in ourthe technical report\footnote{https://github.com/selectmdp}.
\subsection{Single Item Selection}
To see the impact of I-sharing, we consider the CS task with $K=1$, $U=1$, and $\cC = \{\text{stay}\}$, and compare ISQ-I with a vanilla DQN and a sorting DQN.
Fig.~\ref{fig:k1n52050} illustrates the learning performance of the algorithms for $N = 5, 20$, and $50$.

\myparagraph{Impact of I-sharing}
The vanilla DQN performs well when $N=5$, but it fails to learn when $N=20$ and $50$ due to the lack of considering equi-invariance in IS-MDP. Compared to the vanilla DQN, the sorting DQN learns better policies under large $N$ by reducing the state space through sorting. However, ISQ-I still outperforms the sorting DQN when $N$ is large. This result originated from the fact that sorting DQN is affected a lot by the choice of the sorting rule. In contrast, ISQ-I exploits equi-invariance with I-shared Q-network so it can outperform the other baselines for all $N$s especially when $N$ is large.
The result coincides to our mathematical analysis in Theorem~\ref{thm:local optimal} and Theorem~\ref{thm:universal} which guarantee the expressive power of  I-shared Q-network for IS-MDP.

\subsection{Multiple Item Selection}
To exploit the symmetry in the tasks, we apply I- sharing to all the baselines.
For CS task, the environment settings are $K=6, |\mathcal{C}| = 5$, $U=0$ and $N=50,200$.
For PP task, we test with $10$ predators ($N=10$) and $4$ preys in a $10 \times 10$ grid world for $K=4,7,10$.
The learning curves in both CS task (Fig.~\ref{fig:K6}) and PP task (Fig.~\ref{fig:predator-prey})
clearly show that ISQ-I outperforms the other baselines (except other ISQ variants) in most of the scenarios even though we modify all the baselines to apply I-sharing.
This demonstrates that ISQ successfully considers the requisites for S-MDP or IS-MDP: a combination of the selected items, command assignment, and future state after the combinatorial selection. 

\myparagraph{Power of ISQ: Proper Selection}
Though I-shared Q-networks give the ability to handle large $N$ to all the baselines,
ISQs outperform all others in every task.  
This is because only ISQ can handle all the requisites to compute correct Q-values.
IDQN and RSQ perform poorly in many tasks since they do not smartly select the items.
RSQ performs much worse than ISQ when $K \ll N$ in both tasks since it only focuses on assigning proper commands but not on selecting good items.
Even when $K = N$ (Fig.~\ref{fig:K10N10}), ISQ-I is better than RSQ since RSQ needs to explore all combinations of selection, while ISQ-I only needs to explore specific combinations.
The other baselines show the importance of future prediction, action selection, and full observation.
First, MQ shares the parameters like ISQ-I, but it only considers a reward for the current state.
Their difference in performance shows the gap between considering and not considering future prediction in both tasks.
In addition, ISQ-I (1) only needs to select items but still has lower performance compared to ISQ-I.
This shows that ISQ-I is able to exploit the large action space.
Finally, EQ estimates Q-functions using each item's information.
The performance gap between EQ and ISQ-I shows the effect of considering full observation in calculating Q-values.

\myparagraph{Impact of P-sharing}
By sharing the parameters in the beginning, ISQ-P learns significantly faster than ISQ-I
in all cases as illustrated by the learning curves in Fig.~\ref{fig:K6} and \ref{fig:predator-prey}.
ISQ-P also outperforms ISQ-U in the PP task because of the increase in the number of parameters at the end of the training process.
With these advantages, ISQ-P achieves two goals at once: fast training in early stage and good final performances.
\myparagraph{Power of ISQ: Generalization Capability}
Another advantage of ISQ is powerful generality under environments with different number of items,
which is important in real situations.
When the number of items changes, a typical Q-network needs to be trained again.
However, ISQ has a fixed number of parameters $\vert \btheta \vert = O(K)$ regardless of $N$.
Therefore, we can re-use the trained $\theta_{k}$ for an item size $N_{tr}$ to re-construct another model for a different item size $N_{te}$.
From the experiments of ISQ-P on different CS scenarios, we observe that for the case $N_{tr}=50, N_{te}=200$, ISQ-P shows an $103\%$ performance compared to $N_{tr}=200, N_{te}=200$.
In contrast, for the case $N_{tr}=200 $ and $N_{te}=50$, it shows an $86\%$ performance compared to $N_{tr}=50$ and $N_{te}=50$.
These are remarkable results since the numbers of the items are fourfold different ($N=50, 200$).
We conjecture that ISQ can learn a policy efficiently in an environment with a small number of items and transfer the knowledge to a different and more difficult environment with a large number of items.
\section{Conclusion}
In this paper, we develop a highly efficient and scalable algorithm to solve continual combinatorial selection by converting the original MDP into an equivalent MDP and leveraging two levels of weight sharing for the neural network. We provide mathematical guarantees for the expressive power of the weight shared neural network. Progressive-sharing share additional weight parameters among $K$ cascaded Q-networks. We demonstrate that our design of progressive sharing outperforms other baselines in various large-scale tasks.


\section*{Acknowledgements}
\thanks{This work was supported by Institute for Information \& communications Technology Planning \& Evaluation(IITP) grant funded by the Korea government(MSIT) (No.2016-0-00160, Versatile Network System Architecture for Multi-dimensional Diversity).} 

\thanks{
This research was supported by Basic Science Research Program through the National Research Foundation of Korea(NRF) funded by the Ministry of  Science and ICT (No. 2016R1A2A2A05921755).}


\bibliographystyle{named}
\bibliography{ijcai19}
\appendix
\onecolumn
\begin{appendices}

\section[cross]{Intra-Parameter Sharing} \label{sec:IntraSharing}
\subsection[cross]{Single Channel} \label{sec:IntraSharing_math_single_channel}
In this section, we formally redefine the two types of the previously defined weight shared layers $\phi_{k}(\cdot)$ and $\psi_{k}(\cdot)$ with the EI property, i.e., for all $\sigma_{s}:=(\sigma_{x}, \sigma_{i})\in \bS_{k} \times \bS_{N-k}$, 
\begin{align*}
  \phi_k(\sigma_s(\bx,\bi)) = \sigma_s(\phi_k(\bx,\bi)), \qquad
  \psi_k(\sigma_s(\bx,\bi)) = \sigma_i(\psi_k(\bx,\bi)).
\end{align*}

We start with the simplest case when $\phi_{k}: \Real^{\vert \bx \vert } \times \Real^{\vert \bi \vert} \rightarrow \Real^{\vert \bx \vert} \times \Real^{\vert \bi \vert}$ and $\psi_{k}:\Real^{\vert \bx \vert} \times \Real^{\vert \bi \vert}\rightarrow \Real^{\vert \bi \vert}$.
This case can be regarded as a state $s= (\bx, \bi)$ where $\bx=(x_{1}, \cdots, x_{k}) \in \Real^{k}$ and $\bi=(i_{1}, \cdots, i_{N-k}) \in \Real^{N-k}$ in Section~\ref{sec:isq}.  
Let $\mathbf{I}_x \in \mathbb{R}^{k \times k }$ and $\mathbf{I}_i \in \mathbb{R}^{(N-k) \times (N-k) }$ are the identity matrices. 
We  denote
$\mathbf{1}_{x,i} \in \mathbb{R}^{k \times (N-k)}$, $\mathbf{1}_{i,x} \in \mathbb{R}^{(N-k) \times k}$, $\mathbf{1}_{x,1} \in \mathbb{R}^{k \times 1}$, 
and $\mathbf{1}_{i,1} \in \mathbb{R}^{(N-k) \times 1}$
are the matrices of ones. 

\myparagraph {Layer $\phi_{k}$}
Let $\phi_{k}(\bx,\bi):=(\bX, \bI)$ with $\bx, \bX \in \Real ^{k}$ and $ \bi,  \bI \in \Real ^{N-k}$ where the output of the layers $\bX $ and $ \bI$ are defined as  
\begin{align} \label{eq:phi_A_output_XY}
  \tbX:= \rho( \bW_{\tx}\tbx + \bW_{\tx, \tx}\tbx +\bW_{\tx, i}\bi + \bb_{\tx}), \qquad
  \bI:= \rho( \bW_{i}\bi + \bW_{i, i}\bi +\bW_{i, \tx}\tbx + \bb_{i} )
\end{align}
with a non-linear activation function $\rho$.
The parameter shared matrices $\bW_{x}, \cdots, \bW_{i,i}$ defined as follows:
\begin{align*}
   \bW_{\tx} & := W_{\tx} \mathbf{I}_{x} , &   \bW_{\tx, \tx} & :=\frac{W_{\tx, \tx}}{\vert \bx \vert} \mathbf{1}_{x,x}, &   \bW_{\tx, i} & :=\frac{W_{\tx, i}}{\vert \bi \vert} \mathbf{1}_{x, i}, &   \bb_{\tx}&:= b_{\tx} \mathbf{1}_{x, 1}, \\
   \bW_{i}&:=W_{i} \mathbf{I}_{i}, &   \bW_{i,i}&:=\frac{W_{i,i}}{\vert \bi \vert} \mathbf{1}_{i, i} ,&   \bW_{i, \tx} & :=\frac{W_{i, \tx}}{\vert \bx \vert} \mathbf{1}_{i, x}, &   \bb_{i} & :=b_{i} \mathbf{1}_{i, 1}.
\end{align*}
The entries in the weight matrices $\bW_{\tx}, \cdots, \bb_{i}$ are tied by real-value parameters
$W_{\tx}, \cdots, b_{i} \in \Real$, respectively.
Some weight matrices such as $\bW_{\tx,\tx}, \bW_{\tx,i}, \bW_{i,\tx}, \bW_{i,i}$ have normalizing term $1 / {\vert \bx \vert}$ ($=1/k$) or $1 / \vert \bi \vert$ ($=1/ {(N-k)}$).
In our empirical simulation results, these normalizations help the stable training as well as increase the generalization capability of the Q-networks.

\myparagraph{\bf Layer $\psi_{k}$}
The only difference of $\psi_{k}$ from $\phi_{k}$ is that the range of $\psi_{k}(\tbx, \bi)$ is restricted in $\bI$ of (\ref{eq:phi_A_output_XY}), i.e.,  $\psi_{k}(\tbx,\bi) :=\bI \in \Real ^{N-k} $ where $ \bI = \rho( \bW_{x}\tbx + \bW_{x,x}\tbx +\bW_{x,i}\bi  + \bb_{i}). $
The weight matrices are similarly defined as in the $\phi_{k}$ case:
\begin{align*}
  \bW_{i}:=W_{i} \mathbf{I}_{i},\quad \bW_{i,i}:= \frac{W_{i,i}}{\vert \bi \vert} \mathbf{1}_{i,i} ,\quad \bW_{i,\tx}:=\frac{W_{i,\tx}}{\vert \bx \vert} \mathbf{1}_{i,x}, \quad \bb_{i}:=b_{i} \mathbf{1}_{i, 1}.
\end{align*}

\myparagraph{Deep Neural Network with Stacked Layers} 
Recall that the I-shared network $Q_{\theta}(\cdot \, ; \, \theta_{k})$ is formed as follows:
 $$Q_{\theta}(\cdot \, ; \, \theta_{k}):=\psi_{k} \circ \phi_{k}^{D} \circ \cdots \phi_{k}^{1}(\cdot)$$ where $D$ denotes the number of the stacked mutiple layers belonging to $\phi_{k}$.
Therefore, the weight parameter vector $\theta_{k}$ for  $Q_{\theta}(\cdot \, ; \, \theta_{k})$ consists of $\{ W_{\tx}^{d}, \cdots, b_{i}^{d} \}_{d=1}^{d=D}$ for $\phi_{k}$ and $\{W_{i}, W_{i,i}, W_{i,x}, b_{i}\}$ for $\psi_{k}$. 
In contrast, the projected vector $\omega(\theta_{k})$ consists of high dimenional weight parameter vectors such as  $\{ \bW_{\tx}^{d}, \cdots, \bb_{i}^{d} \}_{d=1}^{d=D}$ for $\phi_{k}$ and $\{\bW_{i}, \bW_{i,i}, \bW_{i,x}, \bb_{i}\}$ for $\psi_{k}$. 
\subsection[cross]{Multiple Channels}
\label{sec:multiplechannels}
\noindent {\bf Multiple Channels.}
In the above section, we describe   simplified versions of the intra-sharing layers $$\phi_{k}: \Real^{\vert \bx \vert} \times \Real^{\vert \bi \vert} \rightarrow \Real^{\vert \bx \vert} \times \Real^{\vert \bi \vert}, \qquad \psi_{k}: \Real^{\vert \bx \vert } \times \Real^{\vert \bi \vert}\rightarrow \Real ^{\vert \bi \vert}.$$
In this section, we extend this to
\begin{equation} \label{eq:EILayers}
  \phi_{k}: \Real^{\vert \bx \vert \cdot P_{x} + \vert \bi \vert \cdot P_{i}} \rightarrow \Real^{\vert \bx \vert \cdot O_{x} + \vert \bi \vert \cdot O_{i}}, \qquad \psi_{k}: \Real^{\vert \bx \vert \cdot P_{x} + \vert \bi \vert \cdot P_{i}} \rightarrow \Real^{\vert \bi \vert \cdot O_{i}}
\end{equation}
where $P_{x}, P_{i}, O_{x}, O_{i}$ are the numbers of the features for the input $\bx, \bi$ and the output $X,I$ of each layer, respectively. 
The role of the numbers is similar to that of channels in convolutional neural networks which increase the expressive power and handle the multiple feature vectors.
This wideness allows more expressive power due to the increased numbers of the hidden nodes, according to the \textit{universial approximatin theorem} \cite{gybenko1989approximation}.
Furthermore, our Theorem~\ref{thm:universal} also holds with proper feature numbers in the hidden layers. 
Without loss of generality, we handle the case for $P_{x} = P_{y} = P$ and $O_{x} = O_{y} = O$.
We use superscripts $\bx^{\langle p \rangle}, \by^{\langle p \rangle} $ and $\bX^{\langle o \rangle}, \bY^{\langle o \rangle}$ for $p \in \{1, \cdots, P\}$ and $o \in \{1, \cdots, O\}$ to denote such channels.
Our architecture satisfies that cross-channel interactions are fully connected.
Layer $\phi_{k}(\tbx, \bi)$ with multiple channels is as follows:
\begin{align*}
  & \tbX^{\langle o \rangle}:= \rho\left( \sum_{p=1}^{P} \left(  \bW^{\langle o, p \rangle}_{\tx}\tbx^{\langle p \rangle} + \bW^{\langle o, p \rangle}_{\tx,  \tx}\tbx^{\langle p \rangle} +\bW^{\langle o, p \rangle}_{\tx, i}\bi^{\langle p \rangle} + \bb^{\langle o \rangle}_{\tx} \right) \right), \\
  & \bI^{\langle o \rangle}:= \rho\left(\sum_{p=1}^{P} \left( \bW^{\langle o, p \rangle}_{i}\bi^{\langle p \rangle} + \bW^{\langle o, p \rangle}_{i, i}\bi^{\langle p \rangle} +\bW^{\langle o, p \rangle}_{i,  \tx}\tbx^{\langle p \rangle} + \bb^{\langle o \rangle}_{i} \right) \right)
\end{align*}
where
\begin{align*}
   \bW^{\langle o,p \rangle}_{\tx}&:=W^{\langle o,p \rangle}_{\tx} \mathbf{I}_{x}, & \quad \bW^{\langle o,p \rangle}_{\tx,\tx}&:=\frac{W^{\langle o,p \rangle}_{\tx,\tx}}{\vert \tbx \vert} \mathbf{1}_{x,x}, & \quad \bW^{\langle o,p \rangle}_{\tx, i}&:=\frac{W^{\langle o,p \rangle}_{\tx, i}}{\vert \bi \vert} \mathbf{1}_{x,i}, & \quad \bb^{\langle o \rangle}_{\tx}&:= b^{\langle o \rangle}_{\tx} \mathbf{1}_{x, 1}, \\
  \bW^{\langle o,p \rangle}_{i}&:=W^{\langle o,p \rangle}_{i} \mathbf{I}_{i}, & \quad \bW^{\langle o,p \rangle}_{i,i} & :=\frac{W^{\langle o,p \rangle}_{i,i}}{\vert \bi \vert} \mathbf{1}_{i, i}, & \quad \bW^{\langle o,p \rangle}_{i, \tx} & :=\frac{W^{\langle o,p \rangle}_{i, \tx}}{\vert \tbx \vert} \mathbf{1}_{i,x}, & \quad \bb^{\langle o \rangle}_{i} & :=b^{\langle o \rangle}_{i} \mathbf{1}_{i, 1}.
 \end{align*}
 Similar to the above cases, the entries in the weight matrices  $\bW^{\langle o,p \rangle}_{\tx}, \cdots, \bb^{\langle o \rangle}_{i}$ are tied together by real-value parameters $W^{\langle o,p \rangle}_{\tx}, \cdots, b^{\langle o \rangle}_{i}$ respectively.
 The weight parameter vector $\theta_{k}$ for  $Q_{\theta}(\cdot \, ; \, \theta_{k})$ with multiple channels consists of $\{ W_{\tx}^{d}, \cdots, b_{i}^{d} \}_{d=1}^{d=D}$ for $\phi_{k}$ and $\{W_{i}, W_{i,i}, W_{i,x}, b_{i}\}$ for $\psi_{k}$. 
 In contrast, the projected vector $\omega(\theta_{k})$ consists of high dimenional weight parameter vectors such as  $\{ \bW_{\tx}^{d}, \cdots, \bb_{i}^{d} \}_{d=1}^{d=D}$ for $\phi_{k}$ and $\{\bW_{i}, \bW_{i,i}, \bW_{i,x}, \bb_{i}\}$ for $\psi_{k}$.
\smallskip

\section{Proofs of the theorems}
\subsection[cross]{Relative Local optimality: Theorem~\ref{thm:local optimal}}
\label{proof:local-optimal}
To simplify the explanation, we only consider the case when phase $k=0$ so $s=(\bi)=(i_{1}, \cdots, i_{N})$ and $Q^{\star}$ is permutation equivariant to the order of $\bi$. 
Furthermore, we consider the case of a single channel described in Section~\ref{sec:IntraSharing_math_single_channel}. 
Therefore, we omit to notate $k$ in this subsection and denote $\sigma$ rather than $\sigma_{i} \in \bS_{N}$.
However, our idea for the proof can be easily adapted to extended cases such as $k>0$ or multiple channels.
To summarize, our goal is to show relative local optimality in Theorem~\ref{thm:local optimal} where the loss function $L_{\Omega}$ is defined as 
$$L_{\Omega}(\omega):=\sum_{\sigma \in \bS_{N}} \sum_{\bi \in B} \Big| Q_{\omega}(\sigma(\bi))-Q^{\star}(\sigma(\bi)) \Big|. $$ 
\myparagraph{Skectch of Proof} 
      To use contradiction, we assume that there exists at least one local minima $\theta^{\star} \in  \Theta$ in the loss function $L_{\Theta}$ for I-shared network $Q_{\theta}$
      while $ \omega({\theta}^{\star}) \in \Omega $ is not a local minima in the loss function $L_{\Omega}$ for non-weight shared network $Q_{\omega}$.
      Therefore,
      there must be a vector $ \omega_{0} \in \Omega$ in  $\Omega$ which makes the directional derivative $D_{ \omega_{0}}L_{\Omega} (\omega(\theta^{\star}))<0$.
    We first extend the definition of each $\sigma\in \bS_{N}$ to the corresponding mapping $\sigma:\Omega \rightarrow \Omega$.
    We can generate $N!$ more derivative vector $\sigma(\omega_{0})$ for each $\sigma$ such that $D_{ \sigma(\omega_{0})}L_{\Omega}(\omega(\theta^{\star})) = D_{\omega_{0}}L_{\Omega}(\omega(\theta^{\star}))<0$.
    Therefore, the sum of the whole permuted vectors $ \overline{\omega}:=\sum_{\sigma \in \bS_{N}}  \sigma(\omega_{0}) $ is also a negative derivative vector while belongs to $\omega(\Theta)$ since $ \overline{\omega}$ has the effect of I-sharing from the summation of the all permuted derivative vectors. 
    This fact guarantees the existence of a derivative vector $ \ovt \in \Theta$ such that $\overline{\omega} = \omega(\ovt)$ and 
    $ D_{\ovt} L_{\Theta}(\theta)<0$ and contradicts to the aformentioned assumption that  $\theta^{\star}$ is the local optimal minima of $L_{\Theta}$. 

\myparagraph{Extended Definition for $\sigma \in \bS_{N}$}
\begin{figure}
\centering
  \subfloat[ $Q_{\omega}$]{\includegraphics[width=0.32\linewidth]{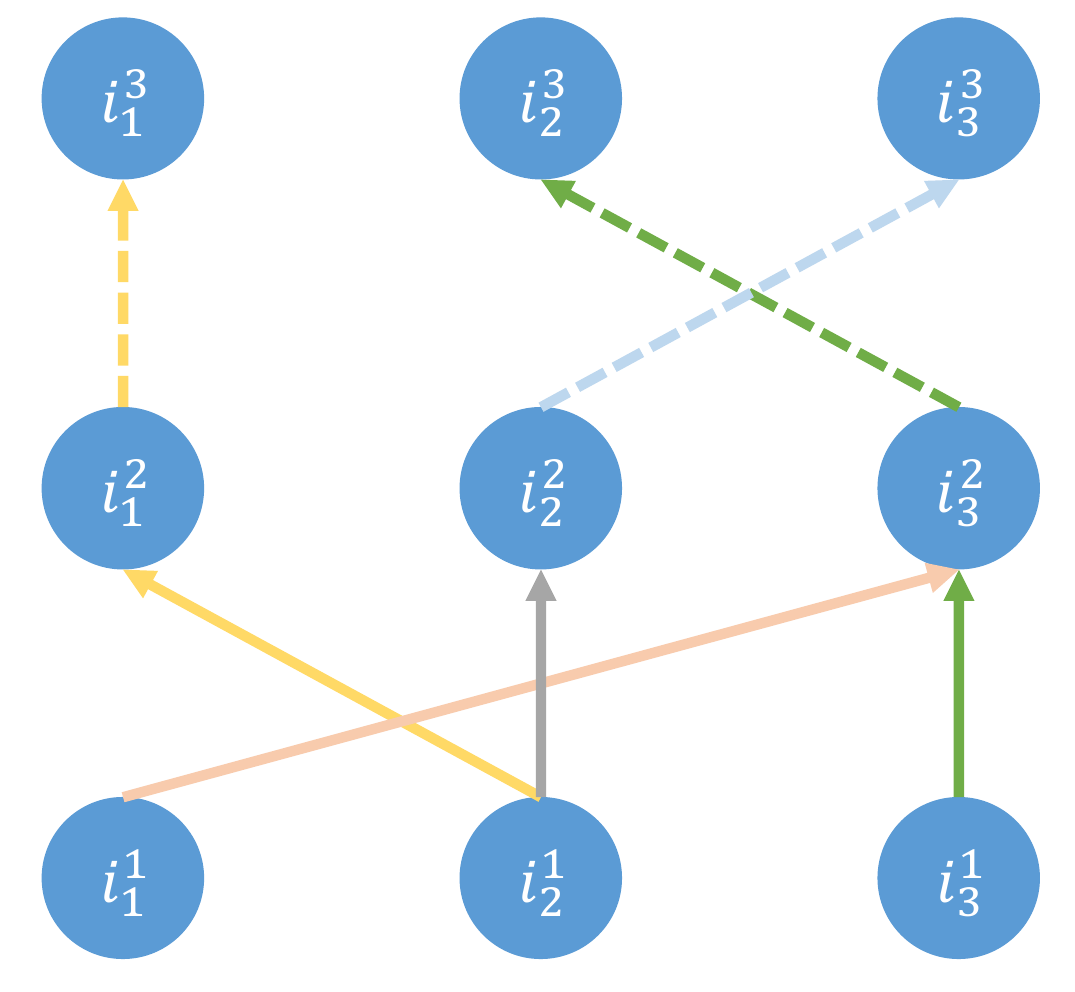}}
  \hspace{2cm}
  \subfloat[ $Q_{\sigma(\omega)}$]{\includegraphics[width=0.32\linewidth]{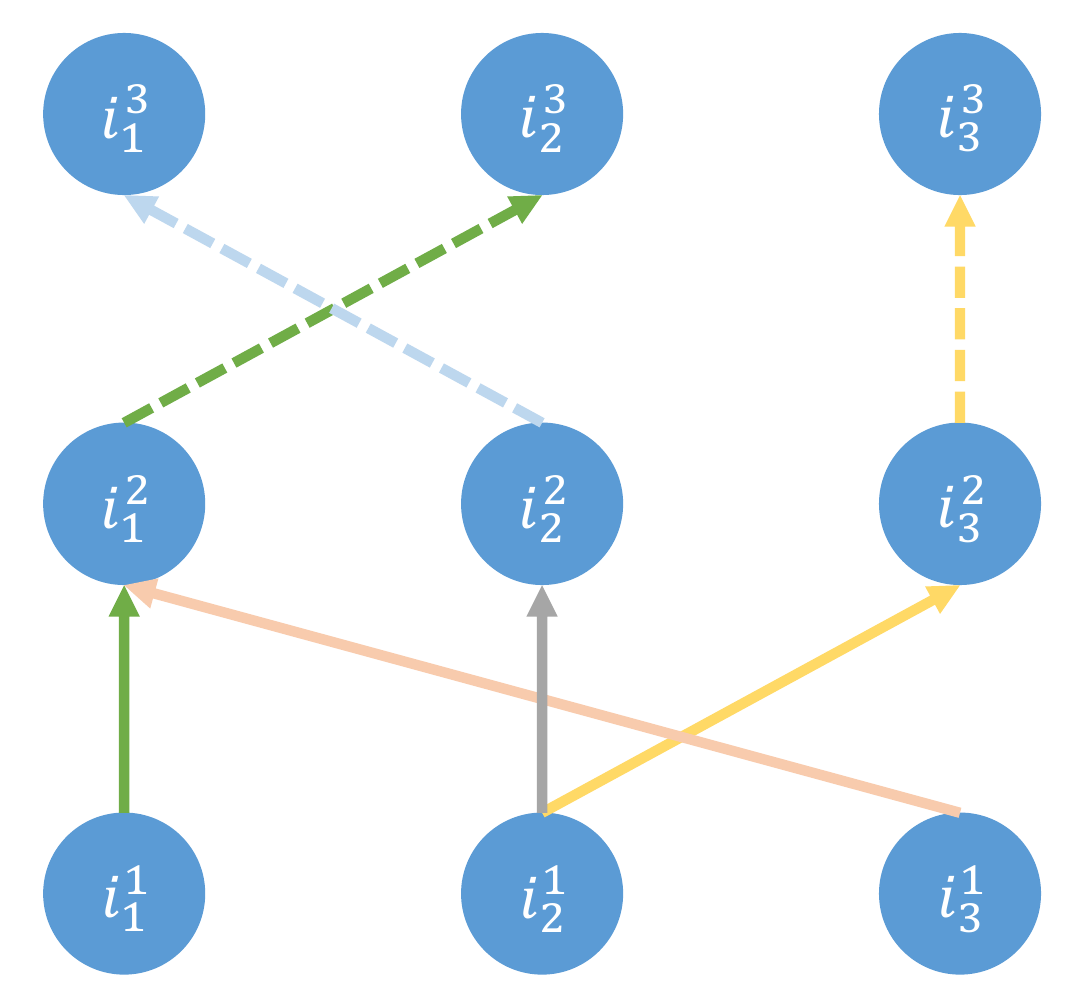}}
  \caption[cross]{The example networks with permuted weight parameter vectors by  $\sigma \in \bS_{3}$ where $\sigma(1) = 3,\sigma(2) = 2, \sigma(3)=1$. 
  If weights in the different network have the same color then they also share the same weight values.}
  \label{fig:Q_permute}
\end{figure}
In this paragraph, we will extend the concept of the permutation $\sigma \in \bS_{N}$ from the original definition on the set $\{1,2,\cdots,N\}$ to the permutation on the weight parameter vector $\omega$ in non-shared weight parameter vector space $\Omega$, i.e.,
$\sigma:\Omega \rightarrow \Omega$
to satisfy the below statement,  
\begin{equation}  \label{eq:Q_permute}
  \forall \sigma\in \bS_{N}, \,  \forall \omega \in \Omega, \,  \forall \bi \in \Real^{N},  \quad \sigma(Q_{\omega}(\bi)) = Q_{\sigma(\omega)}(\sigma(\bi)). 
\end{equation}
To define the permutation with the property in \eqref{eq:Q_permute}, we shall describe how $\sigma$ permutes weight parameters in a layer $\phi_{\omega}:\Real^{N}\rightarrow \Real^{N}$ in $Q_{\omega}$, which can be represented as 
\begin{equation} \label{eq:phi_d_permute}
\phi_{\omega}(\bi) = \bW\bi +\bb 
\end{equation}
where $\bW \in R^{N\times N} $ is a weight matrix and $\bb \in \Real^{N}$ is a biased vector. 
In the permuted layer $\phi_{\sigma(\omega)}$, the weight matrix $\bW$ and $\bb$ in (\ref{eq:phi_d_permute}) convert to  $M_{\sigma} \circ \bW \circ M_{\sigma}^{-1}$ and  $M_{\sigma}\circ \bb$, respectively.
$M_{\sigma} $ is a permutation matrix defined as 
$ M_{\sigma}:=[\be_{\sigma(1)}, \cdots, \be_{\sigma(N)}] $
where $\be_{n}$ is a standard dimensional basis vector in $\Real^{N}$.
With the permuted weights, we can easily see $\sigma(\phi_{\omega}(\bi) ) = \phi_{\sigma(\omega)}(\sigma(\bi))$  for all $\sigma, \omega,$ and $ \bi$.
Therefore, the network $Q_{\sigma(\omega)}$ which is a composite of the layers $\phi_{\sigma(\omega)}$s satisfies \eqref{eq:Q_permute}.
Figure~\ref{fig:Q_permute} describes an example of the permutation on $\omega$. 

Note that the projected weight parameter vector $\omega(\theta)$ for an arbitrary $\theta \in \Theta$ is invariant to the permutation $\sigma:\Omega \rightarrow \Omega$ since $\omega(\theta)$ satisfies the symmetry among the weights from I-sharing, i.e.,
\begin{equation} \label{eq:omega projected}
  \forall \theta \in \Theta, \forall \sigma \in \bS_{N}, \quad  \omega(\theta) = \sigma(\omega(\theta)). 
\end{equation}
  \begin{lemma}[Permutation Invariant Loss Function] \label{thm:Permutation Invariant Loss Function} 
    For any weight parameter vectors $\omega\in \Omega$, $\theta \in \Theta$, and $\sigma\in \bS_{N}$, the below equation holds.   
    \begin{equation} \label{eq:Lomegahomo}
 L_{\Omega}(\omega(\theta)+\omega) = L_{\Omega}(\omega(\theta)+\sigma(\omega)).
    \end{equation}

\begin{proof}[(Proof of Lemma~\ref{thm:Permutation Invariant Loss Function})]
We can derive the result of Lemma~\ref{thm:Permutation Invariant Loss Function} from the below statement. 
  \begin{equation*}
    \begin{split}
      L_{\Omega}(\omega(\theta)+\sigma_{0}(\omega)) & = 
      \sum_{\sigma \in \bS_{N} }  \sum_{\bi \in B} \Big|  Q_{\omega(\theta)+\sigma_{0}(\omega)}(\sigma(\bi)) - Q^{\star}(\sigma(\bi)) \Big|^2  \\
      & =   \sum_{\sigma \in \bS_{N} }  \sum_{\bi \in B} \Big|  Q_{\sigma_{0}(\omega(\theta)+\omega)}(\sigma_{0} \circ \sigma_{0}^{-1} \circ \sigma(\bi)) - Q^{\star}(\sigma_{0} \circ \sigma_{0} ^{-1} \circ \sigma(\bi)) \Big|^2   \qquad\quad (\because(\ref{eq:omega projected})) 
      \\
      &= 
      \sum_{\sigma \in \bS_{N} }  \sum_{\bi \in B} \Big| \sigma_{0}\Big(Q_{\omega(\theta)+\omega}(\sigma_{0}^{-1} \circ \sigma(\bi))\Big) - \sigma_{0}\Big(Q^{\star}(\sigma_{0}^{-1} \circ \sigma(\bi))\Big) \Big|^2   \qquad\qquad (\because  \eqref{eq:equi-inv property}, \eqref{eq:Q_permute}) 
      \\
      &= 
      \sum_{\sigma' \in \bS_{N} }  \sum_{\bi \in B} \Big| Q_{\omega(\theta)+\omega}(\sigma'(\bi)) -  Q^{\star}(\sigma'(\bi)) \Big|^2  \qquad\qquad\qquad\qquad\qquad\qquad\quad (\sigma':=\sigma_{0}^{-1} \circ \sigma) 
      \\
      & = L_{\Omega}(\omega(\theta)+\omega).
    \end{split}
  \end{equation*}
\end{proof}
\end{lemma}

  \begin{proof}[(Proof of Theorem~\ref{thm:local optimal}).]
  We use contradiction by assumping that there exists a local minima $\theta^{\star} \in  \Theta$ of $L_{\Theta}$  while $ \omega({\theta}^{\star}) \in \Omega $ is not a local minima of $L_{\Omega}$.
  Since $\omega({\theta}^{\star}) $ is not local minima of $L_{\Omega}$, there exists a vector $\omega_{0} \in \Omega$ such that the directional derivative of $L_{\Omega}(\omega(\theta^{\star}))$ along $\omega_{0}$ is negative, i.e., 
  $ D_{\omega_{0}}(L_{\Omega}(\omega(\theta^{\star}))) < 0$.
We can find $N!$ additional vectors which have a negative derivative by permuting the $\omega_{0} \in \bS_{N}$ and exploiting the result of Lemma~\ref{thm:Permutation Invariant Loss Function}.
  \begin{equation*}
    \begin{split}
      D_{\sigma(\omega_{0})}(L_{\Omega}(\omega(\theta^{\star}))) & = \lim _{h \rightarrow 0} \frac{L_{\Omega}(\omega(\theta^{\star}))+h \sigma(\omega_{0})) - L_{\Omega}(\omega(\theta^{\star}))}{h} \\
      & = \lim _{h \rightarrow 0} \frac{L_{\Omega}(\omega(\theta^{\star})+h\omega_{0}) - L_{\Omega}(\omega(\theta^{\star}))}{h} \qquad \qquad \qquad (\because \eqref{eq:Lomegahomo}) \\
      & = D_{\omega_{0}}(L_{\Omega}(\omega(\theta^{\star})))<0.
    \end{split}
  \end{equation*}
  The existence of the above limit can be induced from  the differentiability of the activation function $\rho$.
  Furthermore,
the activation function is continuously differentiable,
so if we set  $\overline{\omega}:=\sum_{\sigma \in \bS_{N}} \sigma(\omega_{0})$,
$$D_{\overline{\omega}}(L_{\Omega}(\omega(\theta^{\star})) ) = \sum_{\sigma \in \bS_{N}} D_{\sigma(\omega_{0})}(L_{\Omega}(\omega(\theta^{\star}))) < 0. $$
  From the symmetricity of $\overline{\omega}$ due to the summation of the $N!$ permuted vectors,
  there exists a vector $\ovt \in \Theta $ such that
    $\overline{\omega} =\omega(\overline{\theta})$.
  Thus, $D_{\ovt}(L_{\Omega}(\omega(\theta^{\star}))) =D_{\overline{\omega}}(L_{\Omega}(\omega(\theta^{\star})) ) <0$ which contradicts to the assumption that $\theta^{\star}$ is the local minima on the loss function $L_{\Omega}$. 
  \end{proof}

  \subsection[cross]{Proof of Theorem~\ref{thm:universal}} \label{sec:thm_univ}
  \begin{figure}
    \includegraphics[width=0.5\linewidth]{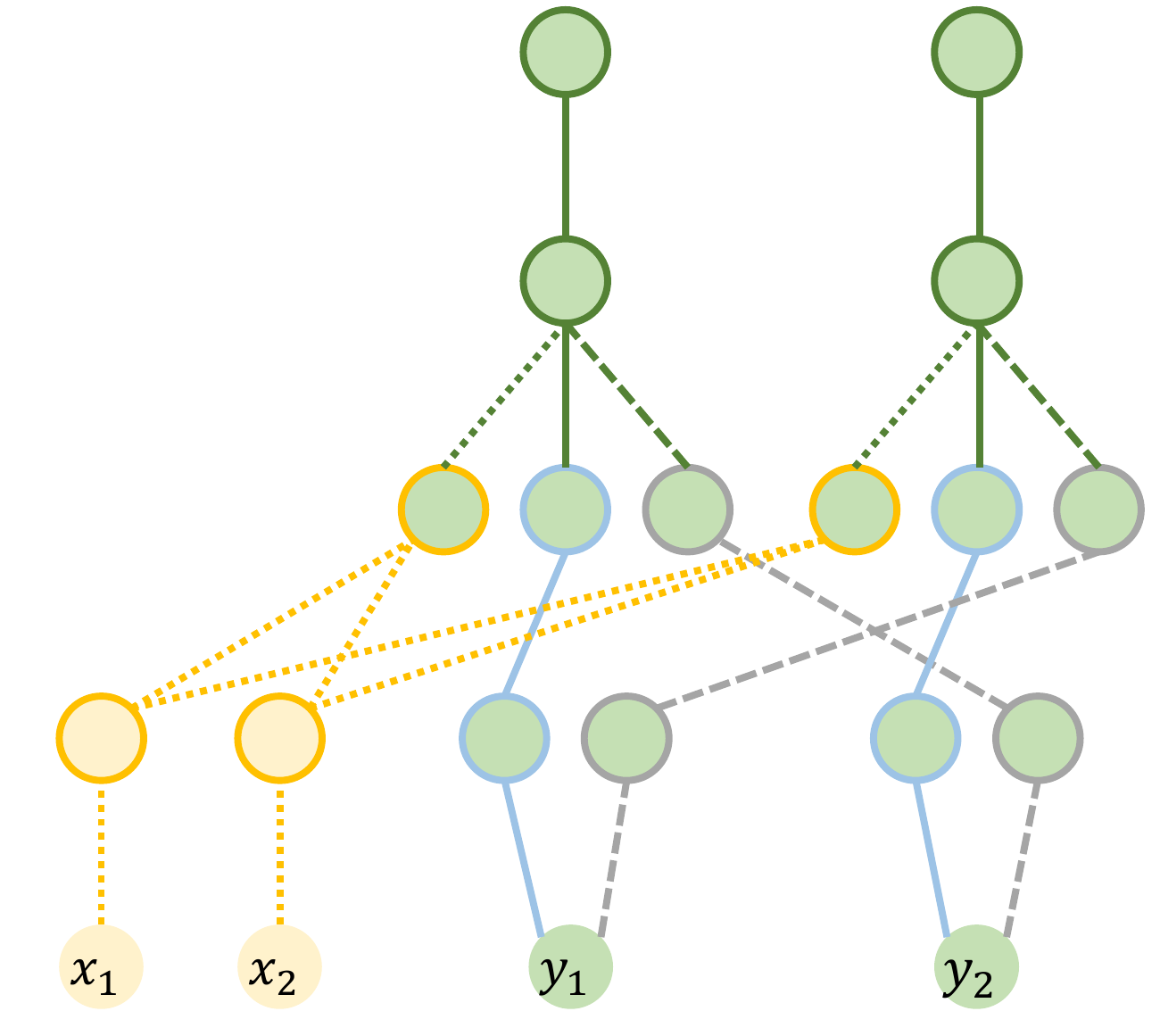}
    \centering
    \caption[cross]{A simplified version of I-shared Q-network $Q_{\theta}(\bx,\bi)$ when $N=4$ and $k=2$ to approximate $Q^{\star}(\bx,\bi)= 
    H (\sum_{x \in \bx} \xi_{x}(x), i_{j}, \sum_{i \in \bi_{-}} \xi_{i}(i))$.
    If the edges share the same color and shape in the same layer, the corresponding weight parameters are tied together. 
    The yellow dotted lines represent a mapping to approximate  $\xi_{x}(x)$.
    The blued solid lines represent an identity mapping.
    The grey dashed lines represent $\xi_{i}(i)$.
    Finally, the green edges generate a mapping to approximate $H$.
    }
    \label{fig:universe}
    \end{figure}

\myparagraph{Sketch of Proof}
We denote $\cX:=\cI \times \cC$ as the domain of the information of the selected items $\bx$.
Recall I-shared Q-network $Q_{\theta}(\bx, \bi):\cX^{k} \times \cI^{N-k} \rightarrow \Real^{(N-k) \times C}$ and the optimal Q-function $Q^{\star}(\bx, \bi):\cX^{k} \times \cI^{N-k} \rightarrow \Real^{(N-k) \times C}$ for each phase $k$ share the same input and output domain.
We denote $[Q_{\theta}(\bx,\bi)]_{j} \in \Real^{C}$ and $[Q^{\star}(\bx,\bi)]_{j}\in \Real^{C}$ as the $j$th row of output of $Q_{\theta}(\bx,\bi)$ and $Q^{\star}(\bx,\bi)$ respectively for $1 \leq j \leq N-k$. 
In other words,  
$$
Q_{\theta}(\bx,\bi)=
\begin{bmatrix}
  [Q_{\theta}  (\bx,\bi)]_{1} \\
   \cdots \\
  [Q_{\theta}  (\bx,\bi)]_{N-k}
\end{bmatrix}.
$$ 

In this proof, we will show that each $[Q^{\star}(\bx, \bi)]_{j}$ can be approximated by $[Q_{\theta}(\bx,\bi)]_{j}$. 
From the EI property of $Q^{\star}(\bx,\bi)$, the $j$th row $[Q^{\star}(\bx,\bi)]_{j}: \cX ^{k} \times \cI^{N-k} \rightarrow \Real ^{C}$
is permutation invariant to the orders of the elements in $\bx$ and $\bi_{-}:=(i_{1}, \cdots,i_{j-1}, i_{j+1}, \cdots,  i_{N-k})$ respectively, i.e.,
  \begin{equation}
   \forall \sigma_{x} \in \bS_{k}, \text{ } \forall \sigma_{i_{-}} \in \bS_{N-k-1}, \quad [Q^{\star}(\bx,i_{j}, \bi_{-})]_{j} \equiv [Q^{\star}(\sigma_{x}(\bx), i_{j}, \sigma_{i_{-}}(\bi_{-}))]_{j}.
  \end{equation}
  In Lemma~\ref{thm:lem_conti_repr}, we show that $[Q^{\star}(\bx,i_{j}, \bi_{-})]_{j}$ can be decomposed in the form of $ H(\sum_{x \in \bx} \xi_{x}(x), i_{j}, \sum_{i \in \bi_{-}}\xi_{i}(i))$ where $H, \xi_{x}, \xi_{y}$ are proper continuous functions.
  Finally, we prove that I-shared Q-network $Q_{\theta}$ with more than four layers can approximate the decomposed forms of the functions: $H, \xi_{x},$and $\xi_{y}$.

\begin{lemma} \label{thm:lem_conti_repr}
  If a continuous function $F(\bx,i,\bi_{-}):\cX^{k} \times \cI \times \cI^{N-k-1} \rightarrow \Real^{C}$ is permutation invariant to the orders of the items in $\bx \in \cX^{k} $ and $\bi_{-} \in \cI^{N-k-1}$, i.e.,
  $$ \forall \sigma_{x} \in \bS_{k}, \text{ }\forall  \sigma_{i_{-}} \in \bS_{N-k-1}, \quad F(\sigma_{x}(\bx),i, \sigma_{i_{-}}(\bi_{-})) \equiv F(\bx,i,\bi_{-}).  $$
  if and only if $F(\bx,i,\bi_{-})$ can be represented by proper continous functions  $H, \xi_{x},$ and $ \xi_{i}$ with the form of
  \begin{equation} \label{eq:Hxyz}
    F(\bx,i,\bi_{-})= H\Big(\sum_{x \in \bx}\xi_{x}(x), i, \sum_{i \in \bi_{-}}\xi_{i}(i)\Big).
  \end{equation}
  \begin{proof}
  The sufficiency is easily derived from the fact that $\sum_{x \in \bx}\xi_{x}(x)$, and $\sum_{i \in \bi_{-}}\xi_{i}(i)$ are permutation invariant to the orders of $\bx$ and $ \bi_{-}$ respectively.
  Therefore,  $H\Big(\sum_{x \in \bx}\xi_{x}(x), i, \sum_{i \in \bi_{-}}\xi_{i}(i)\Big)$ must be permutation invariant to the orders of $\bx$ and $\bi_{-}$.

  To prove the necessity, we exploit a result of Theorem 7 in \cite{zaheer2017deep} about the existences of following continuous functions with proper compact sets $\cX_{0} $ and $ \cI_{0}$ on Euclidean space. 
  \begin{equation}
    \begin{array}{llr}
       \exists \eta_{x}:\cX_{0}^{k+1} \rightarrow \cX^{k}, & \exists \xi_{x}:\cX \rightarrow \cX_{0}^{k+1},  & \qquad \eta_{x}(\sum_{x \in \bx} \xi_{x}(x)) := \bx, \\  \exists \eta_{i}: \cI_{0}^{N-k} \rightarrow \cI^{N-k-1}, & \exists \xi_{i}:\cI \rightarrow  \cI_{0}^{N-k},  & \qquad \eta_{i}(\sum_{i \in \bi} \xi_{i}(i)) := \bi.
    \end{array}
    \end{equation}
  Therefore, we can define a continuous function $H(\cdot,\cdot,\cdot): \cX_{0}^{k+1} \times \cI \times \cI_{0}^{N-k} \rightarrow \Real^{C}$ as
   $$H(\cdot,\cdot,\cdot):= F(\eta_{x}(\cdot), \, \cdot \, , \eta_{i}(\cdot)).$$
   It is obvious that the function $H$ satisfies \eqref{eq:Hxyz}.
  \end{proof}
\end{lemma}

\begin{proof}
With the result of the lemma, the only remained problem to be checked is that I-shared Q-network $Q_{\theta}(\bx,i_{j}, \bi_{-})$ with 4 layers is able to approximate $H(\sum_{x \in \bx} \xi_{x}(x), i_{j}, \sum_{y \in \bi_{-}}\xi_{y}(y))$ if the size of the nodes increases.
During this proof, we use the universal approximation theorem by \cite{gybenko1989approximation} which shows that any continuous function $f$
on a compact domain can be approximated by a proper 2-layered neural network.  
To approximate functions of the decomposition, we can increase the number of the channels described in Section~\ref{sec:multiplechannels}.
We omit the biased term $\bb$ for simplicity. 
Figure~\ref{fig:universe} describes the architecture of $Q_{\theta}$.
For $\xi_{x}$, there exist weight parameter vectors $M$ and $M'$ in $\theta$ such that $ \xi_{x} \approx  M \circ M'$. 
We set $W_x^{1}:=M'$ and $W_{x,i}^{2}:=M$ (Apricot edges).  
Similarly, we can also find weight parameter vectors $W_{i}^{1}:=\begin{bmatrix}
    \bm{I} & \bm{0} \\
    \bm{0} & R' \\
  \end{bmatrix} $ 
  and $ W_{i,i}^{2}:=R$ where $\xi_{i} \approx R \circ R'$ (Grey edges). 
  The identity in $W_{i}^{1}$ and $W_{i}^{2}=\bm{I}$ (Blue edges) represent the passing $i_{j}$ as the input of $H$.  
We set $W_{i}^{3} $ and $ W_{i}^{4}$ to satisfy $H \approx W_{i}^{4} \circ W_{i}^{3}$ (Green edges). 
Other weight parameters such as $W^{3}_{i,i}$ just have zero values.  
With this weight parameter vector for $Q_{\theta}$, $Q_{\theta}(\bx,i_{j}, \bi_{-j})$ successfully approximates the function $H\Big(\sum_{x \in \bx}\xi_{x}(x), i, \sum_{i \in \bi_{-}}\xi_{i}(i)\Big)=[Q^{\star}(\bx, i, \bi_{-})]_{j}$ which is $j$th row values of $Q^{\star}$.
Furthermore, the EI property also implies that for all $j$, $[Q_{\theta}(\bx,i, \bi_{-})]_{j}$ are the same function, in fact. 
Therefore, the I-shared Q-network $Q_{\theta}$ with this architecture can approximate all the rows of $Q^{\star}$ simultaneously.
\end{proof}

\section[cross]{Detailed Experiment Settings}  \label{sec:experiment}
In this subsection, we explain the environment settings in more detail.

\subsection[cross]{Evaluation Settings} \label{sec:Evaluation Settings}
\paragraph{Circle Selection (CS)}
\begin{figure}
  \centering
  \subfloat[$r<0$]{\label{fig:negative}\includegraphics[width=0.20\linewidth]{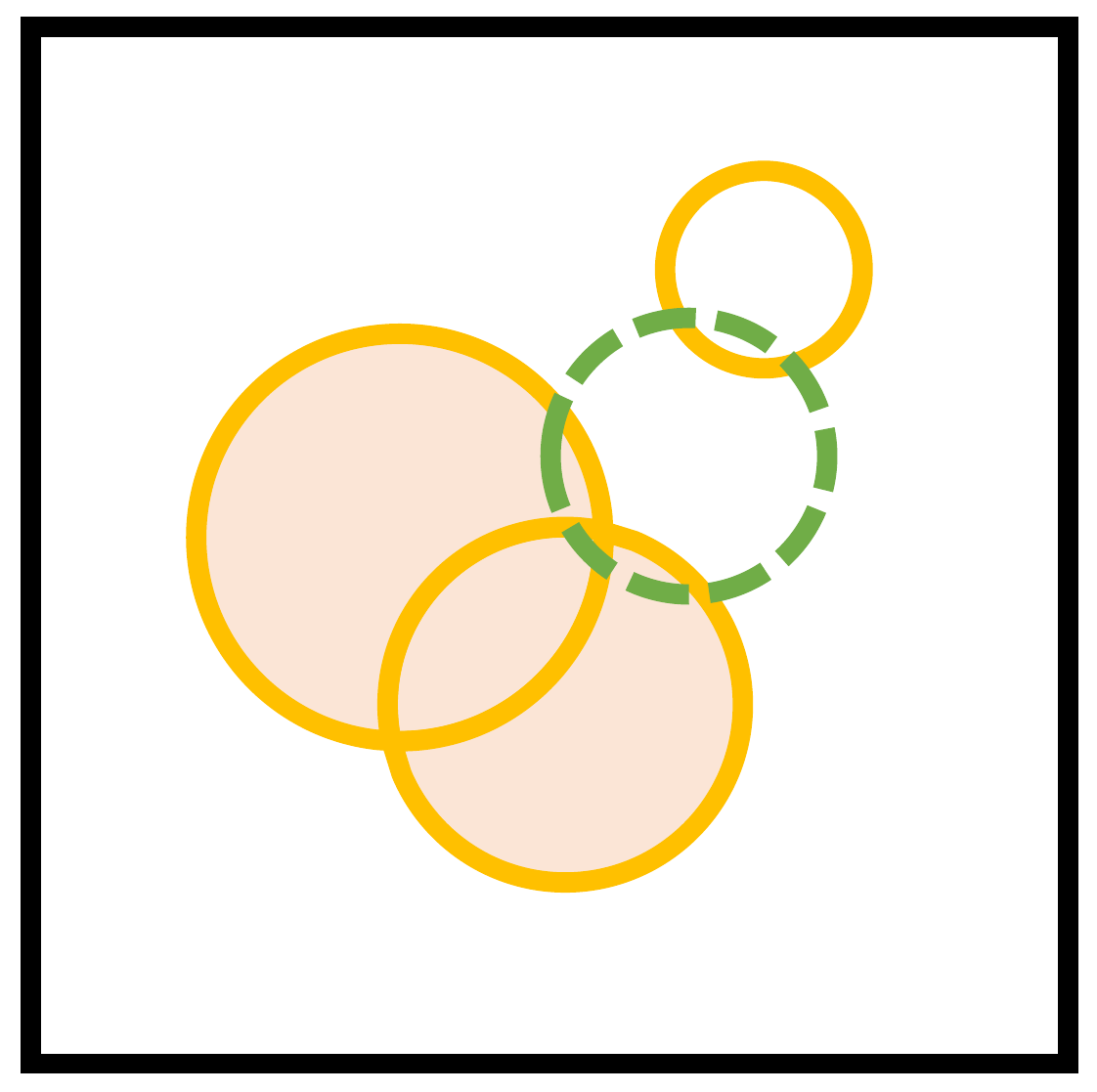}}
  \subfloat[$r=0$]{\label{fig:zero}\includegraphics[width=0.20\linewidth]{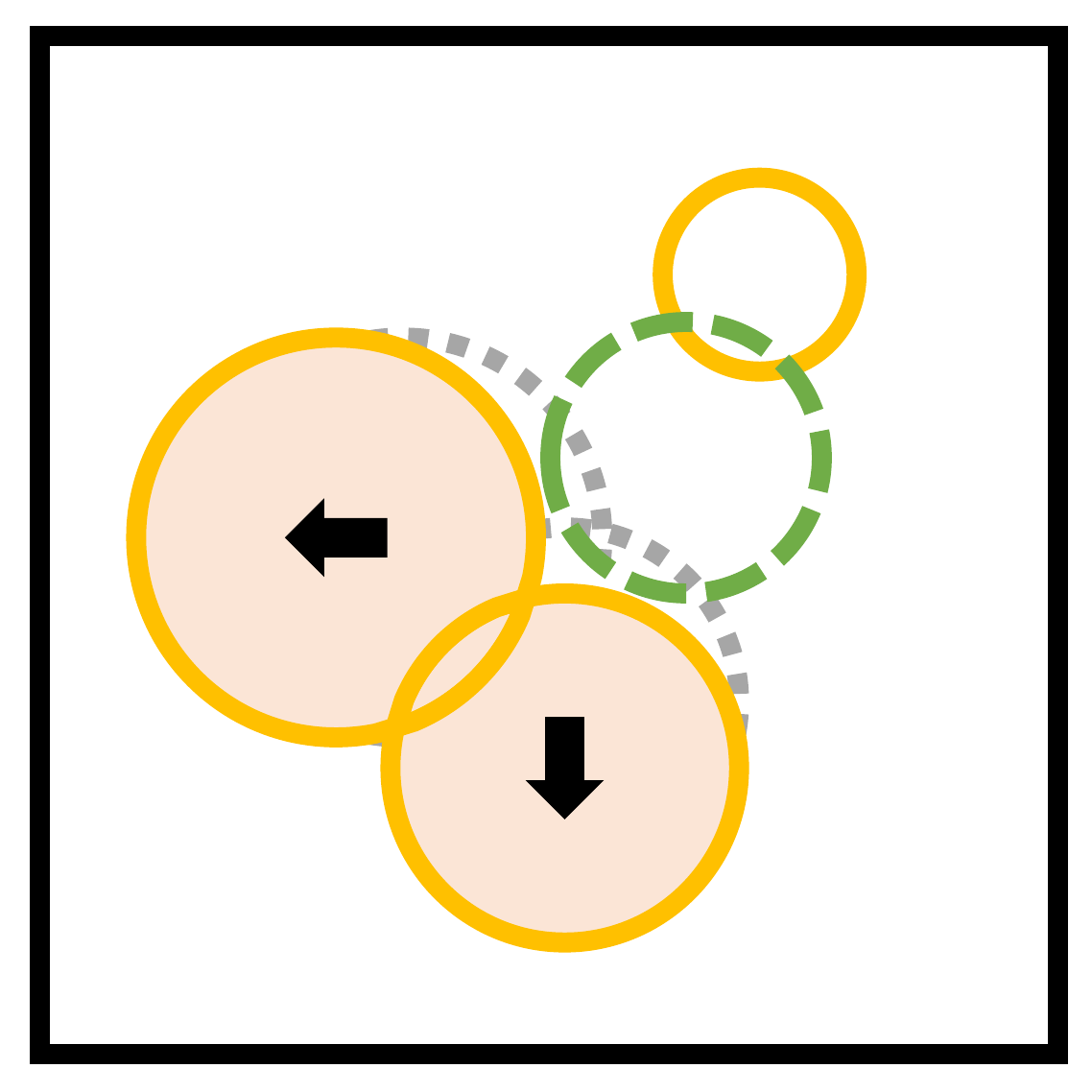}}
  \subfloat[$r>0$]{\label{fig:positive}\includegraphics[width=0.20\linewidth]{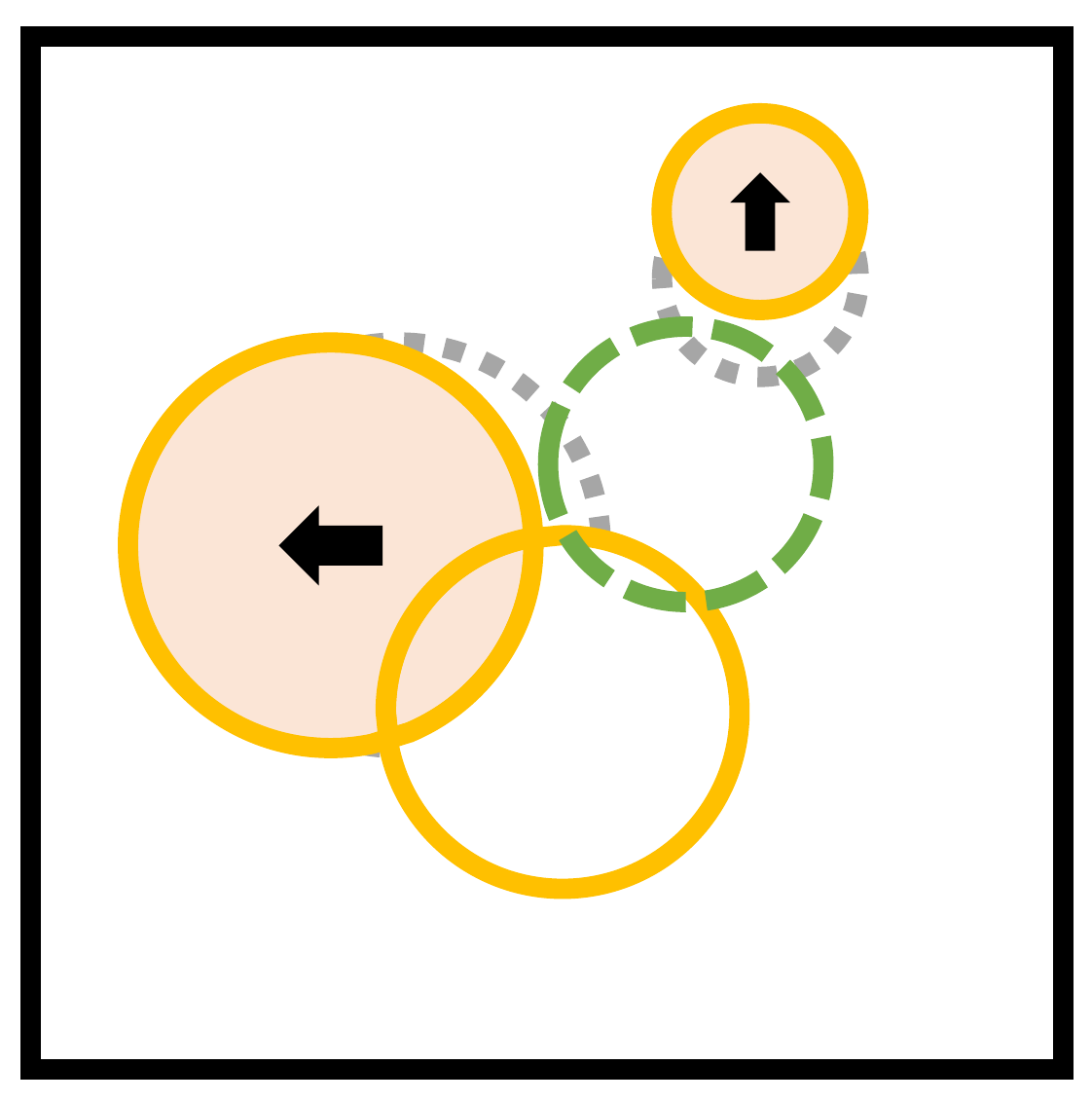}}
  \vspace{-0.25cm}
  \caption[cross]{Example scenarios of the CS task with $N=3$ selectable (orange colored) and $U=1$ unselectable (green dashed) circles, with $K=2$ selected (shaded) circles. The assigned commands are represented by the arrows. The agent receives (a) \textit{negative} reward if selected circles overlap with unselectable one; (b) \textit{zero} reward if only selected circles are overlapped with each other; and (c) \textit{positive} reward if there is no overlap.
  } 
  \label{fig:circles}
\end{figure}
As mentioned in Section~\ref{sec:exp-CS}, the game consists of $N$ selectable and $U$ unselectable circles within a $1 \times 1$ square area, as shown in Fig.~\ref{fig:circles}. Here, circles are the items and $i_{n} :=(pos_x, pos_y, rad)$ are their contexts, where $pos_x$ and $pos_y$ are their center coordinates. 
Initially, all circles have random coordinates and radius, sampled from $(pos_x, pos_y) \in  [-0.5, 0.5] \times [-0.5, 0.5]$ and $rad \in [0, 0.45]$ respectively, 
After the agent selects $K$ circles with the allocated commands, transition by S-MDP occurs as follows. 
The selected circles disappear. The unselectable circles that collide with the selected circles disappear. New circles replace the disappeared circles, each of initial radius $0.01$ with uniformly random position in $ [-0.5, 0.5] \times [-0.5, 0.5]$. 
Remaining circles expand randomly by $[0.045, 0.055]$ in radius (until maximum radius $0.45$) and move with a noise sampled uniformly from $[-0.01, 0.01] \times [-0.01, 0.01]$. 
The agent also receives reward $r$ after the $K$th selection, calculated for each selected circle $k$ of area $A_{k}$ as follows: \textit{Case 1.} The selected circle collides with one or more unselectable circle: $r = -A_{k}$. \textit{Case 2.} Not case 1, but the selected circle collides with another selected circle: $r = 0$. \textit{Case 3.} Neither case 1 nor 2: $r = A_{k}$. 
We test our algorithm when $K=6$ with varying $N=50, 200$.
This fact is described in Figure~\ref{fig:circles}.
\subsection{Predator-Prey (PP)}
In PP, $N$ predators and $U$ preys are moving in $G\times G$ grid world. 
After the agent selects $K$ predators as well as the  commands, the transition in S-MDP occurs. 
In our experiments, we tested the baselines when $N=10$, $U=4$ with varying $K=4,7,10$ while $G=10$.
A reward is a number of the preys that are caught by more than two predators simultaneously. 
For each prey, there are at most 8 neighborhood grids where the predator can catch the prey.

\subsection{Intra-sharing with unselectable items}
In real applications, external context information can be beneficial for the selection in S-MDP. 
For instance, in the football league example, the enemy team's information can be useful to decide a lineup for the next match.
ISQ can handle this contextual information easily with a simple modification of the neural network. 
Similar to invariant part for previously selected items (red parts in Fig.~\ref{fig:Perphase_networks}) of I-shared Q-network, we can add another invariant part in the Q-networks for the external context: the information of the unselectable circles (CS) and prey (PP).

\subsection[cross]{Hyperparameters}
\label{sec:hyperparameters}
During our experiment, we first tuned our hyperparameters for CS and applied all hyperparmeters to other experiments. 
The below table shows our hyperparameters and our comments for Q-neural networks. 

\begin{table}[ht]
  \small
  \centering
  \caption[cross]{Training hyperparameters}
  \label{table:hyperparam}
  \begin{tabular}{L{4.5cm}L{1.2cm}L{7cm}}
    \bf{Hyperparameter} & \bf{Value} & \bf{Descriptions} \\ \midrule
    Replay buffer size & $50,000$ & Larger is stable  \\
    Minibatch size & $64$ & Larger performs better \\
    Learning rate (Adam) & $0.001$ & Larger is faster and unstable \\
    Discount factor & $0.99$ & Discount factor $\gamma$ used in Q-learning update \\
    Target network update frequency & $1000$ & The    larger frequency (measured in number of training steps)
  becomes slower and stable \\
    Initial exploration & $1$ & Initial value of $\epsilon$ used in $\epsilon$-greedy exploration \\
    Final exploration & $0.1$ & Final value of $\epsilon$ used in $\epsilon$-greedy exploration \\
    Number of layers & $3$ & The number  of the layers in the Q-network \\
    Number of nodes  & $48$ & The number of channels per each item in a layer \\
    Random seed & $4$ & The number of random seeds for the independent training \\
  \end{tabular}
\end{table}

\subsection{Computation cost}
We test all baselines on our servers with 
Intel(R) Xeon(R) CPU E5-2630 v4 @ 2.20GHz (Cpu).
Our algorithm (ISQ-I) able to run $1.16 10^6$ steps during one day in CS ($N=200$, $K=6$).
Usually, ISQ-I is robust to large $N$ from I-sharing. However, the computation time linearly increases as $K$ grows since the number of the networks should be trained increases large.
This problem will be fixed if we exploit the parallelization with GPUs.

\end{appendices}

\end{document}